%% file: surfing.tex
\documentclass[12pt,pdftex,noinfoline]{imsart}

\RequirePackage[OT1]{fontenc}
\usepackage{amsthm,amsmath,amsfonts,natbib,mathtools,amssymb}
\RequirePackage{hypernat}
\usepackage[noend]{algpseudocode}
\usepackage[ruled,section]{algorithm}
\usepackage{graphicx}
\usepackage{verbatim}
\usepackage{times}
\usepackage[T1]{fontenc}
\usepackage[text={6.0in,8.6in},centering]{geometry}
\usepackage{graphicx}
\usepackage{hyperref}
\usepackage{yhmath}
\usepackage[usenames,dvipsnames]{xcolor}
\definecolor{shadecolor}{gray}{0.9}
\definecolor{shadecolor}{gray}{0.9}
\hypersetup{citecolor=MidnightBlue}
\hypersetup{linkcolor=MidnightBlue}
\hypersetup{urlcolor=MidnightBlue}
\usepackage{enumerate}
\usepackage{fullpage}
\usepackage{tikz}
\usepackage{dsfont}
\usepackage{accents}
\usepackage{amssymb}
\usepackage{amsthm}
\usepackage{txfonts}
\usepackage{bbm}
\usepackage{subcaption}
\usepackage{wrapfig}
\usepackage{afterpage}

\numberwithin{equation}{section}
\newtheorem{theorem}{Theorem}[section]
\newtheorem{lemma}{Lemma}[section]

\newtheorem{assumption}[theorem]{Assumption}

\theoremstyle{remark}

\def\reals{\mathbb{R}}

\let\hat\widehat

\let\tilde\widetilde
\def\({\left(}
\def\){\right)}
\def\[{\left[}
\def\]{\right]}

\let\hat\widehat
\def\htheta{{\hat\theta}}
\def\kdim{{k}}

\def\n{{n}}

\def\reals{{\mathbb R}}
\def\N{{\mathcal N}}
\def\R{{\mathcal R}}
\newcommand{\sign}{\text{sign}}
\newcommand{\eps}{\varepsilon}

\DeclareMathOperator*{\argmin}{argmin}

\begin{document}

\setlength{\parskip}{0.5em}

\begin{frontmatter}
\title{Surfing: Iterative Optimization Over \\ Incrementally Trained Deep Networks}
\runtitle{Surfing}
\begin{aug}
\vskip10pt
\author{\fnms{Ganlin} \snm{Song}\ead[label=e1]{xiaoqian.yang@yale.edu}}
\,
\author{\fnms{Zhou} \snm{Fan}\ead[label=e3]{david.pollard@yale.edu}}
\,
\author{\fnms{John} \snm{Lafferty}\ead[label=e2]{john.lafferty@yale.edu}}
\vskip10pt
\address{
\begin{tabular}{c}
Department of Statistics and Data Science\\
Yale University
\end{tabular}
\\[10pt]
\today\\[5pt]
\vskip10pt
}
\end{aug}
\input{abstract}
\end{frontmatter}

\input{intro}

\input{background}
\input{results}
\input{experiments}
\input{proofs}
\input{discuss}
\input{ack}
\bibliography{surf,reference}
\bibliographystyle{apalike}

\end{document}

%% file: abstract.tex

\begin{abstract}
  We investigate a sequential optimization procedure to minimize the empirical risk functional $f_{\hat\theta}(x) = \frac{1}{2}\|G_{\hat\theta}(x) - y\|^2$ for certain families of deep networks $G_{\theta}(x)$. The approach is to optimize a sequence of objective functions that use network parameters obtained during different stages of the training process.  When initialized with random parameters $\theta_0$, we show that the objective  $f_{\theta_0}(x)$ is ``nice'' and easy to optimize with gradient descent. As learning is carried out, we obtain a sequence of generative networks $x \mapsto G_{\theta_t}(x)$ and associated risk functions $f_{\theta_t}(x)$, where $t$ indicates a stage of stochastic gradient descent during training.  Since the parameters of the network do not change by very much in each step, the surface evolves slowly and can be incrementally optimized. The algorithm is formalized and analyzed for a family of expansive networks. We call the procedure {\it surfing} since it rides along the peak of the evolving (negative) empirical risk function, starting from a smooth surface at the beginning of learning and ending with a wavy nonconvex surface after learning is complete.  Experiments show how surfing can be used to find the global optimum and for compressed sensing even when direct gradient descent on the final learned network fails.
\end{abstract}

%% file: intro.tex

\section{Introduction} \label{intro}

Intensive recent research has provided insight into the performance and mathematical properties of deep neural networks, improving understanding of their strong empirical performance on different types of data. Some of this work has investigated gradient descent algorithms that optimize
the weights of deep networks during learning \citep{du2018gradient,du2018deep,davis:18,li:17,li:18}. In this paper we focus on optimization over the inputs to an already trained deep network in order to best approximate a target data point. Specifically, we consider the least squares objective function
\begin{align*}
f_\htheta(x) = \frac{1}{2} \| G_\htheta(x) - y\|^2
\end{align*}
where $G_\theta(x)$ denotes a multi-layer feed-forward network and $\htheta$
denotes the parameters of the network after training. The network is considered
to be a mapping from a latent input $x\in \reals^{\kdim}$ to an output
$G_\theta(x)\in\reals^\n$ with $\kdim \ll \n$. A closely related objective is to minimize $f_{\theta,  A}(x) = \frac{1}{2} \| A G_\theta(x) - Ay\|^2$ where $A$ is a random matrix.

\cite{hand2017global} study the behavior of the function $f_{\theta_0, A}$ in a compressed sensing framework where $y = G_{\theta_0}(x_0)$ is generated from a random  network with parameters $\theta_0 = (W_1,\ldots, W_d)$ drawn from Gaussian matrix ensembles; thus, the network is not trained. In this setting, it is shown that the surface is very well behaved. In particular, outside of small neighborhoods around $x_0$ and a scalar multiple of $-x_0$, the function $f_{\theta_0, A}(x)$ always has a descent direction.

When the parameters of the network are trained, the landscape of the function $f_\htheta(x)$ can be complicated; it will in general be nonconvex with multiple local optima. Figure~1 illustrates the behavior of the surfaces as they evolve from random networks (left) to fully trained networks (right) for 4-layer networks trained on Fashion MNIST using a variational autoencoder. For each of two target values $y$, three surfaces $x\mapsto -\frac{1}{2} \| G_{\theta_t}(x) - y\|^2$ are shown for different levels of training.

This paper explores the following simple idea. We incrementally optimize a sequence of objective functions $f_{\theta_0}, f_{\theta_1}, \ldots, f_{\theta_T}$ where the parameters $\theta_0, \theta_1, \ldots, \theta_T = \hat \theta$ are obtained using stochastic gradient descent in $\theta$ during training. When initialized with random parameters $\theta_0$, we show that the empirical risk function $ f_{\theta_0}(x) = \frac{1}{2}\|G_{\theta_0}(x) - y\|^2$ is ``nice'' and easy
to optimize with gradient descent. As learning is carried out, we obtain a sequence of generative networks $x \mapsto G_{\theta_t}(x)$ and associated risk functions $f_{\theta_t}(x)$, where $t$ indicates an intermediate stage of stochastic gradient descent during training.  Since the parameters of the network do not change by very much in each step \citep{du2018deep,du2018gradient}, the surface evolves slowly. We initialize $x$ for the current network $G_{\theta_t}(x)$ at the optimum $x^*_{t-1}$ found for the previous network $G_{\theta_{t-1}}(x)$ and then carry out gradient descent to obtain the updated point $x^*_{t} = \argmin_x f_{\theta_t}(x)$.

We call this process {\it surfing} since it rides along the peaks of the evolving (negative) empirical risk function, starting from a smooth surface at the beginning of learning and ending with a wavy nonconvex surface after learning is complete. We formalize this algorithm in a manner that makes it amenable to analysis. First, when $\theta_0$ is initialized so that the weights are random Gaussian matrices, we prove a theorem showing that the surface has a descent direction at each point outside of a small neighborhood. The analysis of \cite{hand2017global} does not directly apply in our case since the target $y$ is an arbitrary test point, and not necessarily generated according to the random network. We then give an analysis that describes how projected gradient descent can be used to proceed from the optimum of one network to the next. Our approach is based on the fact that the ReLU network and squared error objective result in a piecewise quadratic surface. Experiments are run to show how surfing can be used to find the global optimum and for compressed sensing even when direct gradient descent fails, using several experimental setups with networks trained with both VAE and GAN techniques.

\begin{figure*}[t]
  \def\hsk{\hskip-20pt}
\begin{center}
\begin{tabular}{cccc}
\hsk {\small initial network} &
\hsk {\small partially trained network} &
\hsk {\small fully trained network} &
{target $y$} \\
\hsk \includegraphics[width=.31\textwidth]{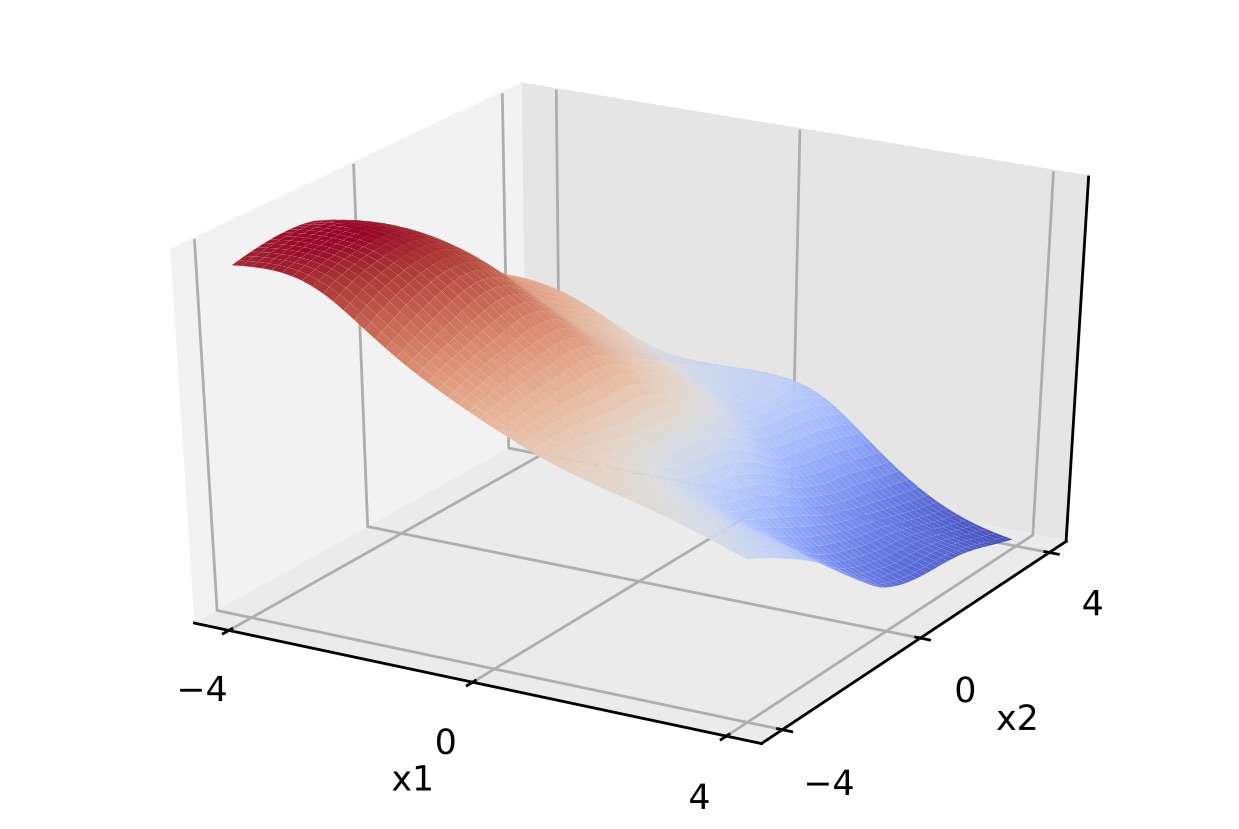} &
\hsk \includegraphics[width=.31\textwidth]{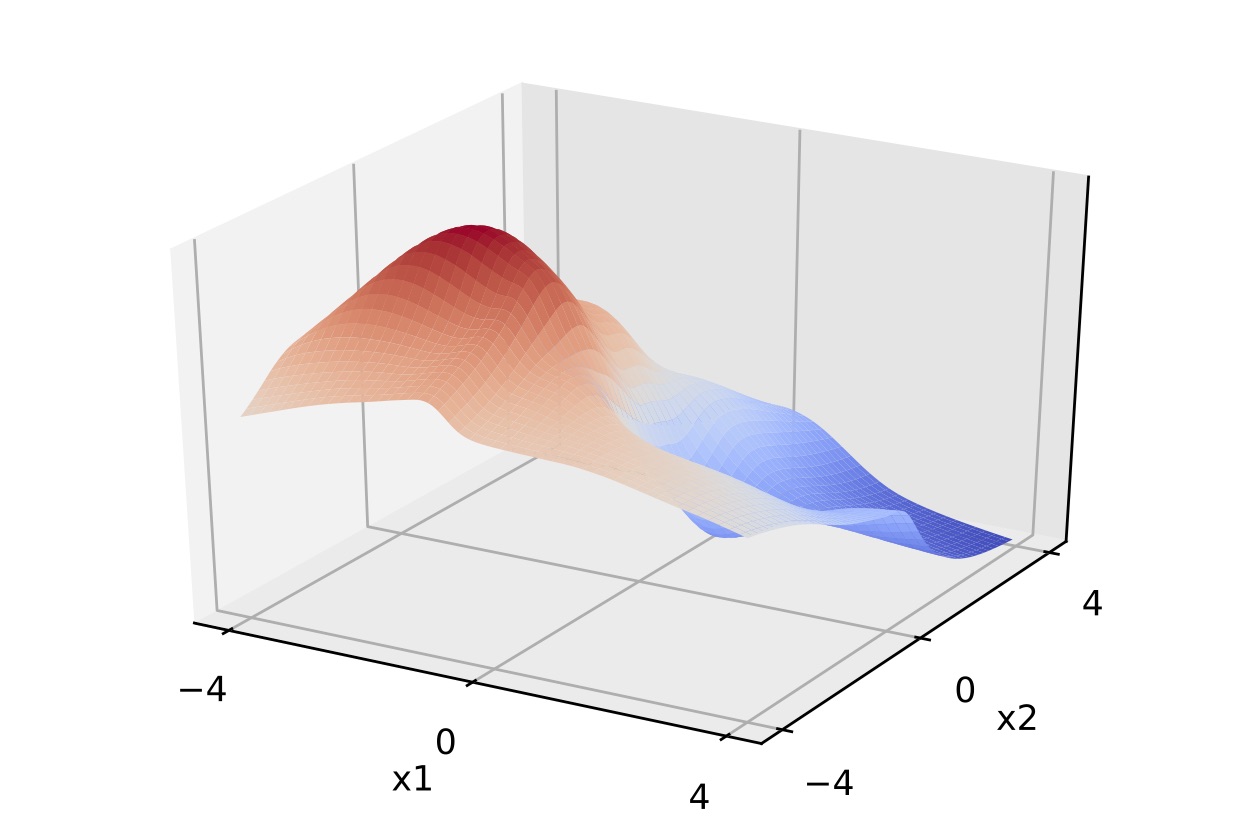} &
\hsk \includegraphics[width=.31\textwidth]{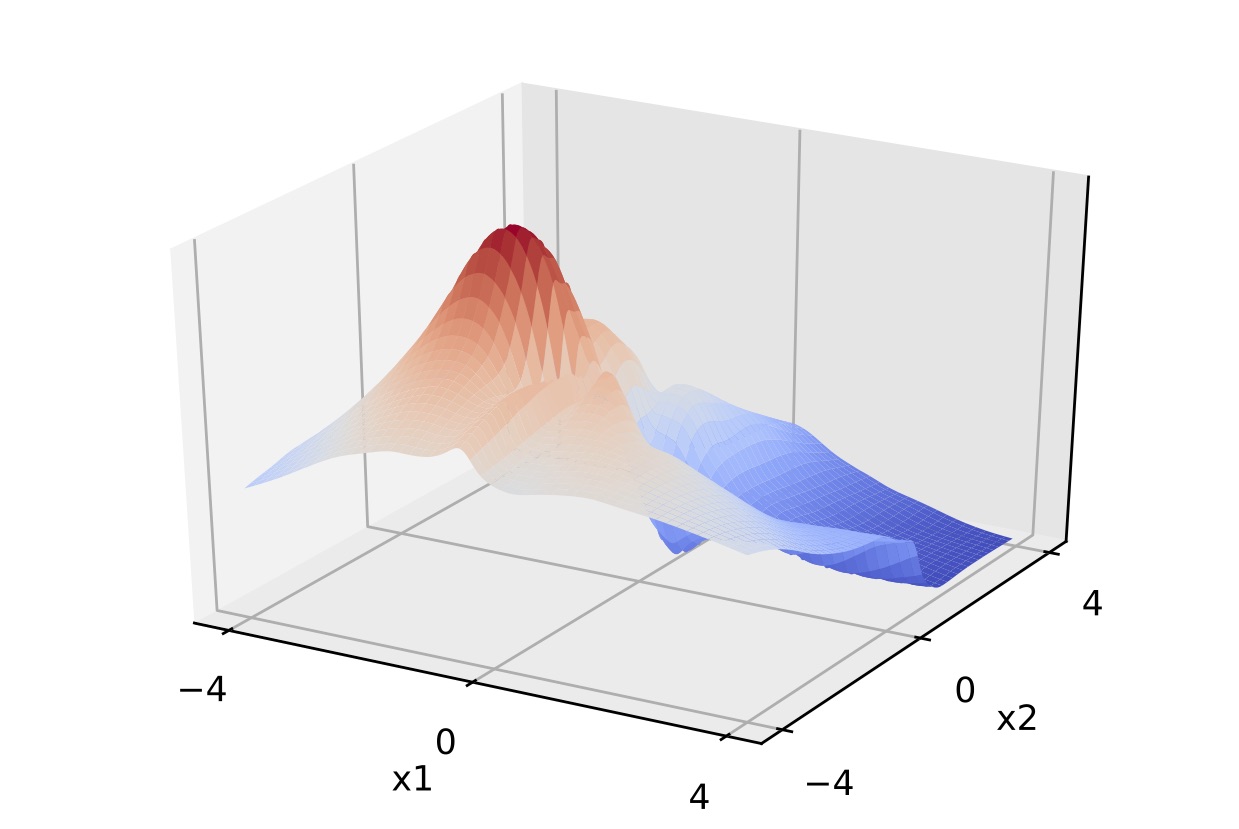} &
\hskip-10pt \raise20pt\hbox{\includegraphics[width=.12\textwidth]{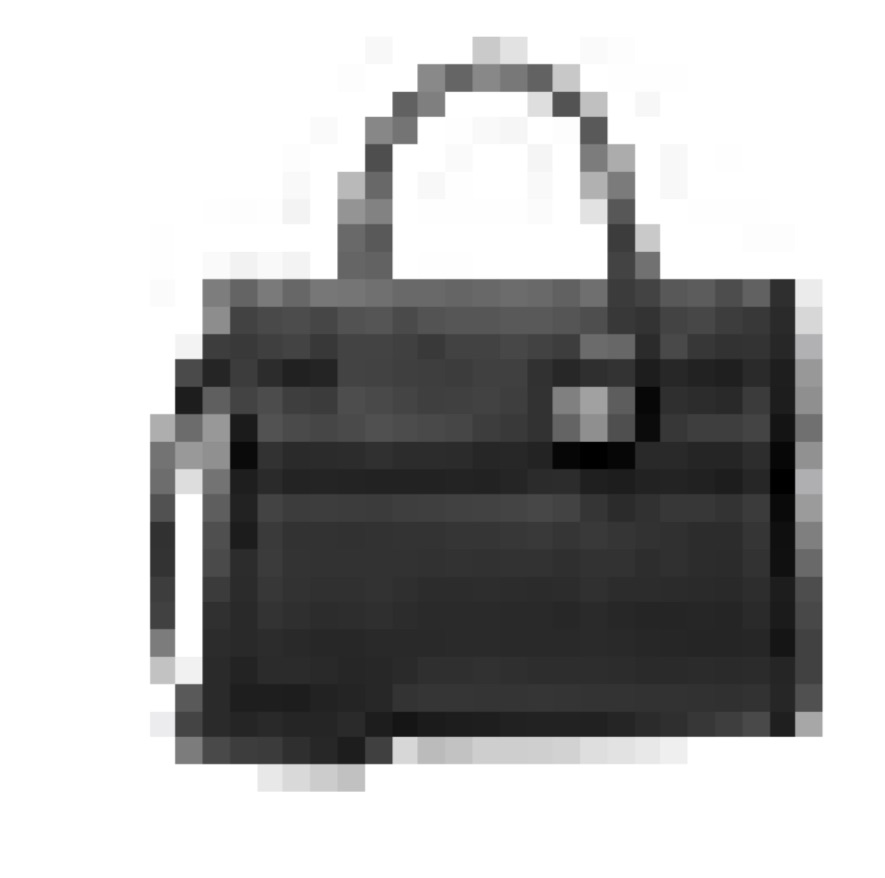}} \\[-1pt]
\hsk \includegraphics[width=.31\textwidth]{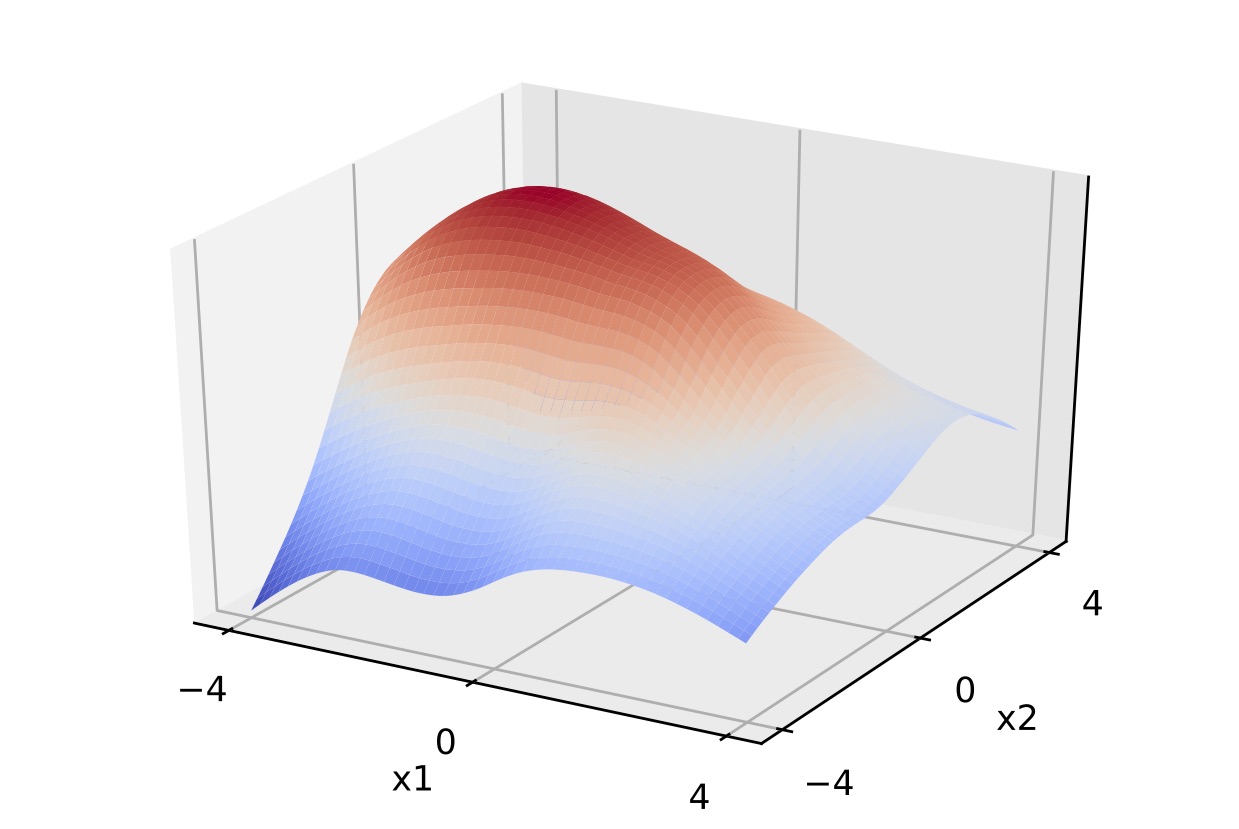} &
\hsk \includegraphics[width=.31\textwidth]{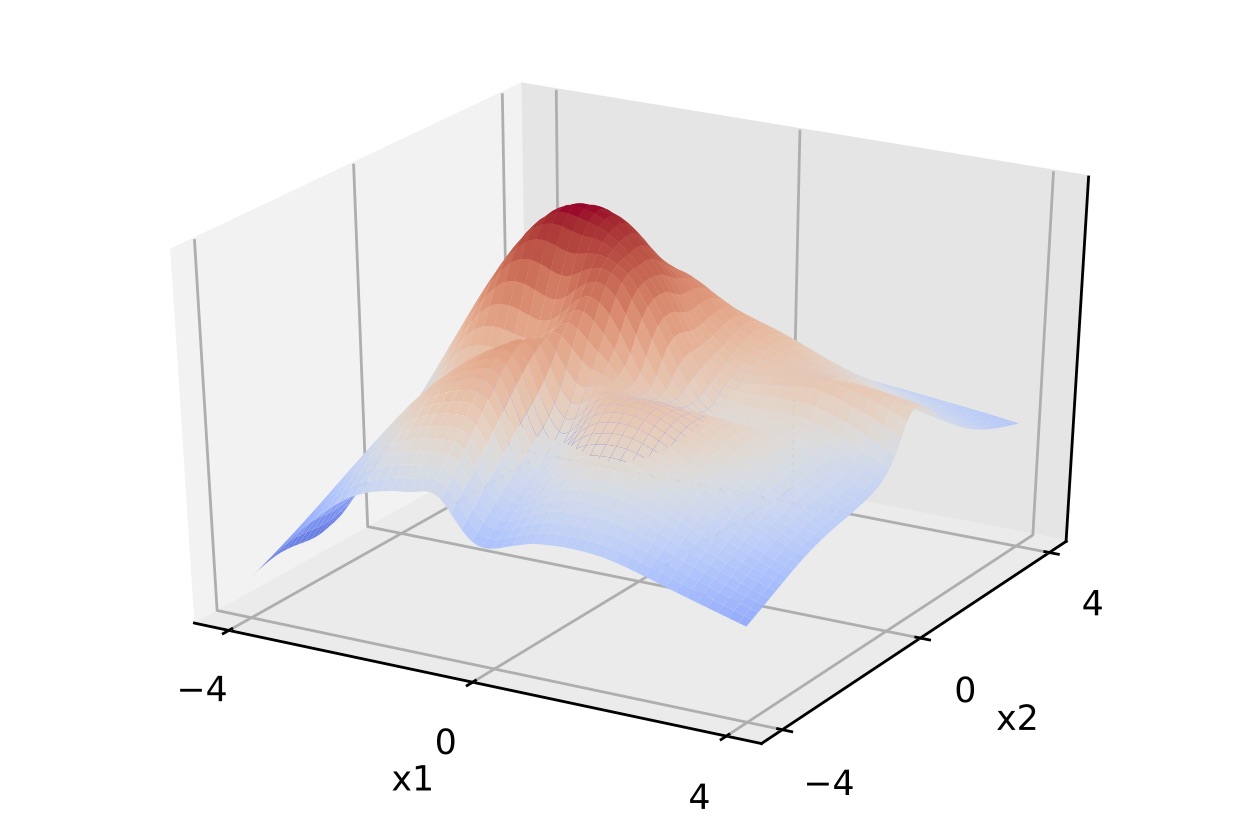} &
\hsk \includegraphics[width=.31\textwidth]{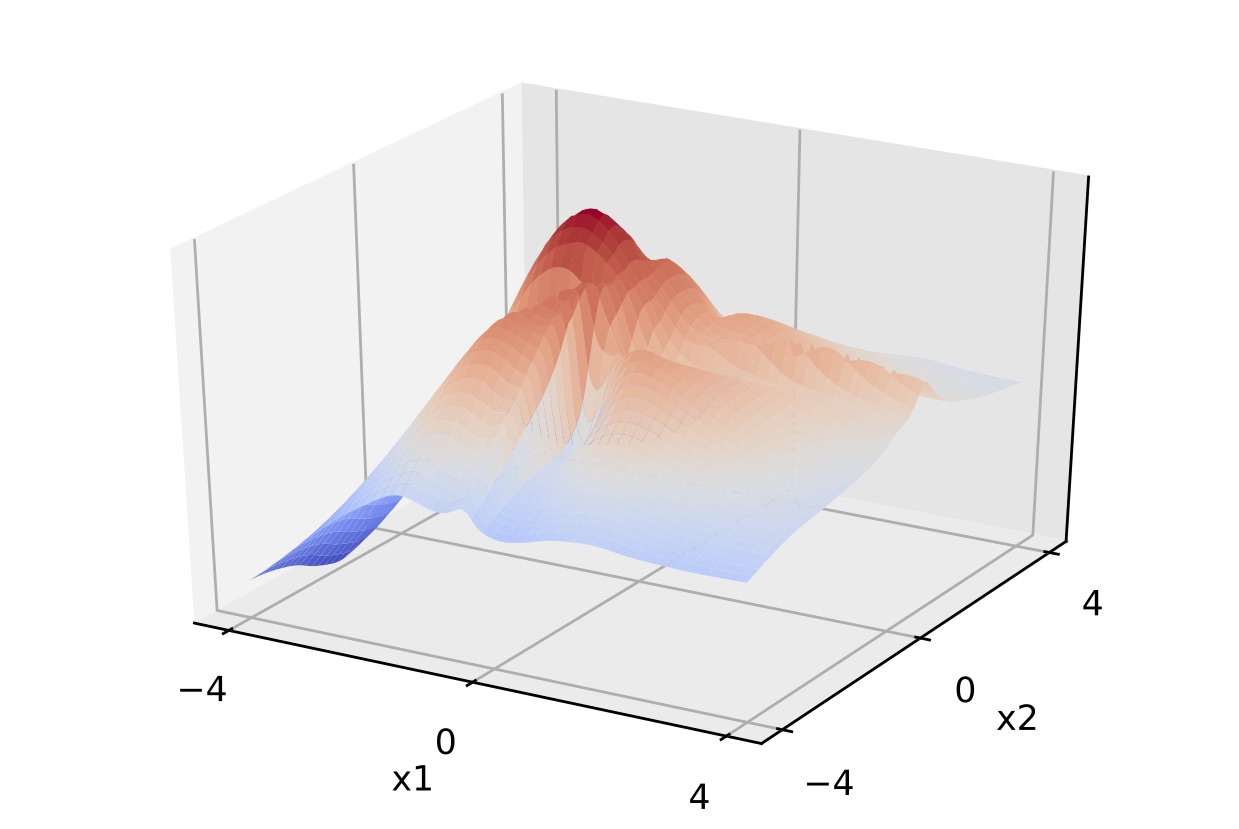} &
\hskip-10pt \raise20pt\hbox{\includegraphics[width=.12\textwidth]{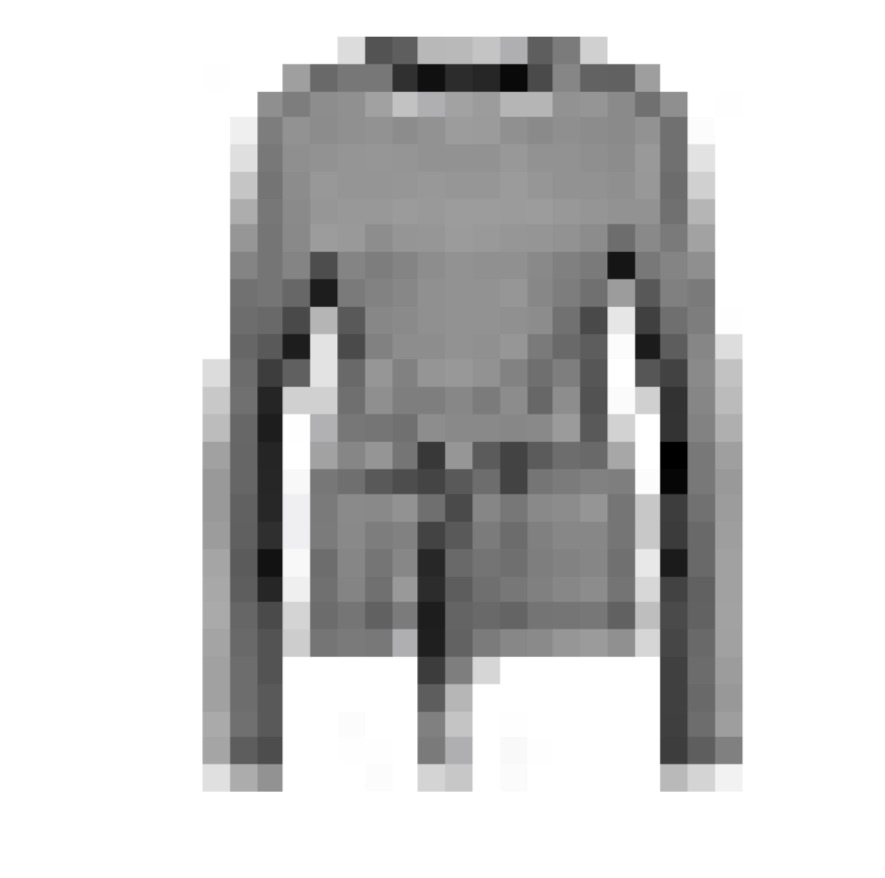}}
\end{tabular}
\end{center}
\caption{Behavior of the surfaces $x\mapsto -{\textstyle \frac{1}{2}} \| G_{\theta_t}(x) - y\|^2$  for two targets $y$ shown for three levels of training,from random networks (left) to fully trained networks (right) on Fashion MNIST data. The network structure
has two fully connected layers and two transposed convolution layers with batch normalization, trained as a VAE.}
\label{figsurf}
\end{figure*}

%% file: background.tex

\def\R{{\mathbb R}}

\section{Background and Previous Results}\label{sec:review}

In this work we treat the problem of approximating an observed vector $y$ in
terms of the output $G_{\hat{\theta}}(x)$ of a trained generative model. Traditional generative processes such as graphical models are statistical models that define a distribution over a sample space. When deep networks are viewed as generative models, the distribution is typically singular, being a deterministic mapping of a low-dimensional latent random vector to a high-dimensional output space. Certain forms of ``reversible deep networks'' allow for the computation of densities and inversion
\citep{nvp,glow,node}.

The variational autoencoder (VAE) approach training a generative (decoder)
network is to model the conditional probability of $x$ given $y$ as Gaussian
with mean $\mu(y)$ and covariance $\Sigma(y)$ assuming that {\it a priori} $x\sim N(0, I_k)$ is Gaussian. The mean  and covariance are treated as the output of a secondary (encoder) neural network. The two networks are trained by maximizing the evidence lower bound (ELBO) with coupled gradient descent algorithms---one for the encoder network, the other for the decoder network $G_{\theta}(x)$ \citep{kingma2013auto}. Whether fitting the networks using a variational or GAN approach \citep{goodfellow2014generative,arjovsky2017wasserstein}, the problem of ``inverting'' the network to obtain
$x^* = \argmin f_{\theta}(x)$ is not addressed by the training procedure.

In the now classical compressed sensing framework
\citep{candes2006stable,donoho2006compressed}, the problem is to reconstruct a
sparse signal after observing multiple linear measurements, possibly with added noise. More recent work has begun to investigate generative deep networks as a replacement for sparsity in compressed sensing.  \cite{bora2017compressed} consider identifying $y=G(x_0)$ from linear measurements $Ay$ by optimizing $f(x) = \frac{1}{2}\|Ay - AG(x)\|^2$.  Since this objective is nonconvex, it is not guaranteed that gradient descent will converge to the true global minimum. However, for certain classes of ReLU networks it is shown that so long as a point
$\hat x$ is found for which $f(\hat x)$ is sufficiently close to zero, then $\|y - G(\hat x)\|$ is also small.
For the case where $y$ does not lie in the image of $G$, an oracle type bound is shown implying that the solution $\hat x$ satisfies
$\|G(\hat x) - y\|^2 \leq C \inf_x \|G(x) - y\|^2 + \delta$ for some small error term $\delta$. The authors observe that in experiments the error seems to converge to zero when $\hat x$ is computed using simple gradient descent; but an analysis of this phenomenon is not provided.

\cite{hand2017global} establish the important result that for a $d$-layer random network and random measurement matrix $A$, the least squares objective has favorable geometry, meaning that outside two small neighborhoods there are no first order stationary points, neither local minima nor saddle points.  We describe their setup and result in some detail, since it provides a springboard for the surfing algorithm. Let $G: \R^k \to \R^n$ be a $d$-layer fully connected feedforward generative neural network, which has the form
$G(x) =  \sigma(W_d...\sigma(W_2 \sigma(W_1x))...)$ where $\sigma$ is the ReLU activation function. The matrix $W_i \in R^{n_i\times n_{i-1}}$ is the set of weights for the $i$th layer and $n_i$ is number of the neurons in this layer with $k=n_0<n_1<...<n_d=n$. If $x_0\in\R^k$ is the input then $AG(x_0)$ is a set of random linear measurements of the signal $y=G(x_0)$. The objective is to minimize
$f_{A,\theta_0}(x) = \frac{1}{2}\big\|AG_{\theta_0}(x) - AG_{\theta_0}(x_0)\big\|^2$
where $\theta_0 = (W_1,\ldots, W_d)$ is the set of weights.

Due to the fact that the nonlinearities $\sigma$ are rectified linear units, $G_{\theta_0}(x)$ is a piecewise linear function. It is convenient to introduce notation that absorbs the activation $\sigma$ into weight matrix $W_i$, denoting
\begin{equation*}
W_{+,x} = \text{diag}(Wx>0)W.
\end{equation*}
For a fixed $W$, the matrix $W_{+,x}$ zeros out the rows of $W$ that do not have a positive dot product with $x$; thus,
$\sigma(Wx) = W_{+,x}x$. We further define $W_{1,+,x} = \text{diag}(W_1x>0)\,W_1$ and
\begin{equation*}
W_{i,+,x} = \text{diag}(W_iW_{i-1,+,x}...W_{1,+,x}x>0)\,W_i.
\end{equation*}
With this notation, we can rewrite the generative network $G_{\theta_0}$ in what looks like a linear form,
\begin{equation*}
G_{\theta_0}(x) = W_{d,+,x}W_{d-1,+,x}...W_{1,+,x}x,
\end{equation*}
noting that each matrix $W_{i,+,x}$ depends on the input $x$. If $f_{A,\theta_0}(x)$ is differentiable at $x$, we can write the gradient as
\begin{equation*}
\nabla f_{A,\theta_0}(x) = \Bigl(\prod_{i=d}^1W_{i,+,x}\Bigr)^T A^TA \Bigl(\prod_{i=d}^1W_{i,+,x}\Bigr)x - \Bigl(\prod_{i=d}^1W_{i,+,x}\Bigr)^T A^TA \Bigl(\prod_{i=d}^1W_{i,+,x_0}\Bigr)x_0.
\end{equation*}
In this expression, one can see intuitively that under the assumption that $A$ and $W_i$ are Gaussian matrices, the gradient $\nabla f_{\theta_0}(x)$ should concentrate around a deterministic vector $v_{x,x_0}$.
\cite{hand2017global} establish sufficient conditions for concentration of the random matrices around deterministic quantities,
so that $v_{x,x_0}$ has norm bounded away from zero if $x$ is sufficiently far from $x_0$ or a scalar multiple of $-x_0$. Their results show that for random networks having a sufficiently expansive number of neurons in each layer, the objective $f_{A,\theta_0}$ has a landscape favorable to gradient descent.

We build on these ideas, showing first that optimizing with respect to $x$ for a random network
and arbitrary signal $y$ can be done with gradient descent. This requires modified proof techniques, since it is no longer assumed that $y=G_{\theta_0}(x_0)$. In fact, $y$ can be arbitrary and we wish to approximate it as $G_{\htheta}(x(y))$ for some $x(y)$. Second, after this initial optimization is carried out, we show how projected gradient descent can be used to track the optimum as the network undergoes a series of small changes. Our results are stated formally in the following section.

%% file: results.tex

\section{Theoretical Results}\label{sec:results}

Suppose we have a sequence of networks
$G_0,G_1,\ldots,G_T$ generated from the training process. For instance, we may take
a network with randomly initialized weights as $G_0$, and record the network after each step of
gradient descent in training; $G_T = G$ is the final trained network.

\renewcommand{\algorithmicrequire}{\textbf{Input:}}
\renewcommand{\algorithmicensure}{\textbf{Output:}}
\setlength{\intextsep}{0pt}
\setlength{\columnsep}{15pt}
\begin{wrapfigure}{r}{0.45\textwidth}
  \begin{minipage}{0.45\textwidth}
  \begin{algorithm}[H]
  \caption{Surfing}
  \begin{algorithmic}[1]
  \Require Sequence of networks $\theta_0,\theta_1,\ldots,\theta_T$
  \State $x_{-1}\leftarrow 0$
  \For{$t = 0$ to $T$}
  \State $x\leftarrow x_{t-1}$
  \Repeat
  \State $x\leftarrow x - \eta\nabla f_{\theta_t}(x)$
  \Until{convergence}
  \State $x_t\leftarrow x$
  \EndFor
  \Ensure $x_T$
  \end{algorithmic}
  \end{algorithm}
  \end{minipage}
  \vskip10pt
\end{wrapfigure}
For a given vector $y \in \R^n$, we wish to minimize the objective
$f(x)=\frac{1}{2}\|AG(x)-Ay\|^2$ with respect to $x$ for the final network $G$, where either $A=I \in \R^{n \times n}$, or
$A \in \R^{m \times n}$ is a measurement matrix with i.i.d.\ $\N(0,1/m)$ entries
in a compressed sensing context. Write
\begin{equation}\label{eq:ft}
f_t(x) = \frac{1}{2}\|AG_t(x) - Ay\|^2, \quad \forall \ t\in[T].
\end{equation}
The idea is that
we first minimize $f_0$, which has a nicer landscape, to obtain the minimizer $x_0$. We then
apply gradient descent on $f_t$ for $t = 1,2,...,T$ successively, starting from
the minimizer $x_{t-1}$ for the previous network.

We provide some theoretical analysis in partial support of this algorithmic
idea. First, we show that at random initialization $G_0$, all critical points of
$f_0(x)$ are localized to a small ball around zero. Second, we show that if
$G_0,\ldots,G_T$ are obtained from a discretization of a
continuous flow, along which the global minimizer of $f_t(x)$ is unique and
Lipschitz-continuous, then a projected-gradient version of surfing
can successively find the minimizers for $G_1,\ldots,G_T$
starting from the minimizer for $G_0$.

We consider expansive feedforward neural networks
$G:\mathbb{R}^k \times \Theta \mapsto \mathbb{R}^n$ given by
\begin{equation*}G(x,\theta)=V\sigma(W_d\ldots \sigma(W_2\sigma(W_1x+b_1)+b_2)\ldots +b_d).\end{equation*}
Here, $d$ is the number of intermediate layers (which we will treat as constant), $\sigma$ is the ReLU activation function
$\sigma(x)=\max(x,0)$ applied entrywise,
and $\theta=(V,W_1,...,W_d,b_1,...,b_d)$ are the network parameters.
The input dimension is $k \equiv n_0$, each intermediate
layer $i \in [d]$ has
weights $W_i \in \R^{n_i\times n_{i-1}}$ and biases $b_i \in\R^{n_i}$, and a linear
transform $V \in \R^{n\times n_d}$ is applied in the final layer.

For our first result, consider fixed $y \in \R^n$
and a random initialization $G_0(x) \equiv G(x,\theta_0)$
where $\theta_0$ has Gaussian entries (independent of $y$).
If the network is sufficiently expansive at
each intermediate layer, then the following shows that
with high probability, all critical points
of $f_0(x)$ belong to a small ball around 0. More concretely, the directional derivative $D_{-x/\|x\|} f_0(x)$ satisfies
\begin{equation}\label{eq:vx}
D_{-x/\|x\|} f_0(x) \equiv \lim_{t \to 0^+} \frac{f_0(x-tx/\|x\|)-f_0(x)}{t}<0.
\end{equation}
Thus $-x/\|x\|$ is a first-order descent direction of the objective
$f_0$ at $x$.

\begin{theorem}\label{thm:init}
Fix $y \in \R^n$.
Let $V$ have $\mathcal{N}(0,1/n)$ entries, let $b_i$ and $W_i$ have
$\mathcal{N}(0,1/n_i)$ entries for each $i \in [d]$, and suppose these
are independent. There exist $d$-dependent constants $C,C',c,\eps_0>0$
such that for any $\eps \in (0,\eps_0)$, if
\begin{enumerate}
\item $n \geq n_d$ and $n_i>C(\eps^{-2}\log \eps^{-1})n_{i-1}\log n_i$ for
all $i \in [d]$, and
\item Either $A=I$ and $m=n$,
or $A \in \R^{m \times n}$ has i.i.d.\ $\N(0,1/m)$ entries
(independent of $V,\{b_i\},\{W_i\}$) where $m \geq Ck(\eps^{-1}\log \eps^{-1})
\log(n_1\ldots n_d)$,
\end{enumerate}
then with probability at least $1-C(e^{-c\eps m}
+n_de^{-c\eps^4 n_{d-1}}+\sum_{i=1}^{d-1} n_i e^{-c\eps^2 n_{i-1}})$,
every $x \in \R^k$ outside the ball $\|x\| \leq C'\eps(1+\|y\|)$
satisfies (\ref{eq:vx}).
\end{theorem}

We defer the proof to Section \ref{sec:proof}.
Note that if instead $G_0$ were correlated with $y$, say $y=G_0(x_*)$ for some
input $x_*$ with $\|x_*\| \asymp 1$,
then $x_*$ would be a global minimizer of $f_0(x)$, and
we would have $\|y\| \asymp \|x_d\| \asymp \ldots \asymp \|x_1\| \asymp
\|x_*\| \asymp 1$ in the above network where $x_i \in \R^{n_i}$ is the output
of the $i^\text{th}$ layer. The theorem shows that for a random
initialization of $G_0$ which is independent of $y$, the minimizer is
instead localized to a ball around 0
which is smaller in radius by the factor $\eps$.

For our second result, consider a network flow
\begin{equation*}G^s(x) \equiv G(x,\theta(s))\end{equation*}
for $s \in [0,S]$, where
$\theta(s)=(V(s),W_1(s),b_1(s),\ldots,W_d(s),b_d(s))$
evolve continuously in a time parameter $s$.
As a model for network training, we assume that $G_0,\ldots,G_T$ are
obtained by discrete sampling from this flow via $G_t=G^{\delta t}$,
corresponding to $s \equiv \delta t$ for a small time discretization step
$\delta$.

We assume boundedness of the weights and uniqueness and Lipschitz-continuity of
the global minimizer along this flow.
\begin{assumption}
\label{assum:cont}
There are constants $M,L<\infty$ such that
\begin{enumerate}
\item For every $i \in [d]$ and $s \in [0,S]$, $$\|W_i(s)\| \leq M.$$
\item The global minimizer $x_*(s)=\argmin_x f(x,\theta(s))$ is unique and satisfies
\begin{equation*}\|x_*(s)-x_*(s')\| \leq L|s-s'|\end{equation*}
  where $f(x,\theta(s))=\frac{1}{2}\|AG(x,\theta(s))-Ay\|^2$.
\end{enumerate}
\end{assumption}

Fixing $\theta$, the function
$G(x,\theta)$ is continuous and piecewise-linear in $x$.
For each $x \in \R^k$, there is at least one linear piece
$P_0$ (a polytope in $\R^k$) of this function
that contains $x$. For a slack parameter $\tau>0$, consider the rows given by
\begin{equation*}
  S(x,\theta,\tau)=\{(i,j):|w_{i,j}^\top x_{i-1}+b_{i,j}| \leq \tau\},
\end{equation*}
where
\begin{equation*}x_{i-1}=\sigma(W_{i-1} \ldots \sigma(W_1x+b_1) \ldots+b_{i-1})\end{equation*}
is the output of the $(i-1)^\text{th}$ layer for this input $x$, and
$v_j^\top$, $w_{i,j}^\top$, and $b_{i,j}$
are respectively the $j^\text{th}$ row of $V$, the $j^\text{th}$ row of $W_i$ and the $j^\text{th}$ entry of
$b_i$ in $\theta$. Define
\begin{equation*}\mathcal{P}(x,\theta,\tau)=\{P_0,P_1,\ldots,P_G\}\end{equation*}
as the set of all linear pieces $P_g$ whose activation patterns differ from
$P_0$ only in rows belonging to $S(x,\theta,\tau)$. That is, for every $x' \in P_g \in
\mathcal{P}(x,\theta,\tau)$ and $(i,j) \notin S(x,\theta,\tau)$, we have
\begin{equation*}\sign(w_{i,j}^\top x_{i-1}'+b_{i,j})=\sign(w_{i,j}^\top x_{i-1}+b_{i,j})\end{equation*}
where $x_{i-1}'$ is the output of the $(i-1)^\text{th}$ layer for input $x'$.

With this definition, we consider a stylized
projected-gradient surfing procedure in Algorithm \ref{alg:projsurf}, where
$\operatorname{Proj}_P$ is the orthogonal projection onto the polytope $P$.

\vskip20pt
\begin{algorithm}
\caption{Projected-gradient Surfing}\label{alg:projsurf}
\begin{algorithmic}[1]
  \Require Network flow
  $\{G(\cdot,\theta(s)):s \in [0,S]\}$, parameters
  $\delta,\tau,\eta>0$.
\State Initialize $x_0=\argmin_x f(x,\theta(0))$.
\For{$t=1,\ldots,T$}
  \For{each linear piece $P_g \in \mathcal{P}(x_{t-1},\theta(\delta t),\tau)$}
    \State $x \leftarrow x_{t-1}$
    \Repeat
      \State $x \leftarrow \operatorname{Proj}_{P_g}(x-\eta \nabla
f(x,\theta(\delta t)))$
    \Until{convergence}
    \State $x_t^{(g)} \leftarrow x$
  \EndFor
  \State $x_t \leftarrow x_t^{(g)}$ for $g \in \{0,\ldots,G\}$ that achieves
  the minimum value of $f(x_t^{(g)},\theta(\delta t))$.
\EndFor
\Ensure $x_T$
\end{algorithmic}
\end{algorithm}
\vskip10pt

The complexity of this algorithm depends on the number of pieces $G$
to be optimized over in each step. We expect this to be small in practice when
the slack parameter
$\tau$ is chosen sufficiently small.

The following shows that for any $\tau>0$, there is a
sufficiently fine time discretization $\delta$ depending on $\tau,M,L$ such that
Algorithm \ref{alg:projsurf} tracks the global minimizer.
In particular, for the final objective $f_T(x)=f(x,\theta(\delta T))$
corresponding to the network $G_T$, the output $x_T$ is the global
minimizer of $f_T(x)$.

\begin{theorem}
\label{thm:surf}
Suppose Assumption \ref{assum:cont} holds. For any $\tau>0$,
if $\delta<\tau/(L\max(M,1)^{d+1})$ and $x_0=\argmin_x
f(x,\theta(0))$, then the iterates $x_t$ in Algorithm \ref{alg:projsurf}
are given by $x_t=\argmin_x f(x,\theta(\delta t))$ for each $t=1,\ldots,T$.
\end{theorem}

\begin{proof}
For any fixed $\theta$, let $x,x' \in \R^k$ be two inputs to $G(x,\theta)$.
If $x_i,x_i'$ are the corresponding outputs of the $i^\text{th}$ layer,
using the assumption $\|W_i\| \leq M$ and the fact that the ReLU activation $\sigma$ is
1-Lipschitz, we have
\begin{align*}
\|x_i-x_i'\|&=\|\sigma(W_ix_{i-1}+b_i)-\sigma(W_ix_{i-1}'+b_i)\|\\
&\leq \|(W_ix_{i-1}+b_i)-(W_ix_{i-1}'+b_i)\|\\
&\leq M\|x_{i-1}-x_{i-1}'\| \leq \ldots \leq M^i\|x-x'\|.
\end{align*}

Let $x_*(s)=\argmin_x f(x,\theta(s))$. By assumption,
$\|x_*(s-\delta)-x_*(s)\| \leq L\delta$. For the network with parameter
$\theta(s)$ at time $s$, let $x_{*,i}(s)$ and $x_{*,i}(s-\delta)$ be the outputs
at the $i^\text{th}$ layer corresponding to inputs $x_*(s)$ and
$x_*(s-\delta)$. Then for any $i \in [d]$ and $j \in [n_i]$, the above yields
\begin{align*}
|(w_{i,j}(s)^\top x_{*,i}(s-\delta)+b_{i,j})&-
(w_{i,j}(s)^\top x_{*,i}(s)+b_{i,j})|
\leq \|w_{i,j}(s)\|\|x_{*,i}(s-\delta)-x_{*,i}(s)\|\\
&\leq M \cdot M^i \|x_*(s-\delta)-x_*(s)\| \leq M^{i+1}L\delta.
\end{align*}
For $\delta<\tau/(L\max(M,1)^{d+1})$, this implies that for every $(i,j)$
where $|w_{i,j}(s)^\top x_{*,i}(s-\delta)+b_{i,j}| \geq \tau$, we have
\begin{equation*}\sign(w_{i,j}(s)^\top x_{*,i}(s-\delta)+b_{i,j})=
\sign(w_{i,j}(s)^\top x_{*,i}(s)+b_{i,j}).\end{equation*}
That is, $x_*(s) \in P_g$ for some $P_g \in
\mathcal{P}(x_*(s-\delta),\theta(s),\tau)$.

Assuming that $x_{t-1}=x_*(\delta(t-1))$, this implies that
the next global minimizer $x_*(\delta t)$ belongs to some
$P_g \in \mathcal{P}(x_{t-1},\theta(\delta t),\tau)$. Since $f(x,\theta(\delta
t))$ is quadratic on $P_g$, projected gradient
descent over $P_g$ in Algorithm \ref{alg:projsurf} converges to
$x_*(\delta t)$, and hence Algorithm \ref{alg:projsurf} yields
$x_t=x_*(\delta t)$. The result then follows from induction on $t$.
\end{proof}

%% file: experiments.tex

\def\R{{\mathbb R}}

\section{Experiments}
\label{sec:experiments}

\afterpage{
\begin{table}[t!]
\begin{center}
  \vskip20pt
\begin{tabular}{|l|ccc|ccc|}
  \hline
Model           & \multicolumn{3}{c|}{VAE} & \multicolumn{3}{c|}{DCGAN} \\
\hline
Input dimension & 5       & 10    & 20    & 5       & 10      & 20     \\
Regular Adam    & 99.3    & 100   & 100   & 50.3    & 69.3    & 67.7   \\
Surfing         & 99.3     & 100   & 100   & 91.3    & 96.7    & 99.3    \\
\hline
\end{tabular}
\vskip10pt
\begin{tabular}{|l|ccc|ccc|}
  \hline
Model           & \multicolumn{3}{c|}{WGAN} & \multicolumn{3}{c|}{WGAN-GP} \\
\hline
Input dimension & 5       & 10    & 20    & 5       & 10      & 20     \\
Regular Adam    & 48.3    & 75.0  & 91.0  & 46.7    & 63.7    & 69.7   \\
Surfing         & 86.7    & 99.3  & 100   & 92.0    & 97.7    & 99.0 \\
\hline
\end{tabular}
\end{center}
\vskip5pt
\caption{Surfing compared against direct gradient descent over final trained network
for different training procedures. Shown are percentages of solutions $\hat x_T$ satisfying $\|\hat x_T - x_*\| < 0.01$.}
\label{table:prop}
\end{table}
\begin{figure}[h!]
  \begin{center}
  \begin{tabular}{ccc}
  \includegraphics[width=.3\textwidth]{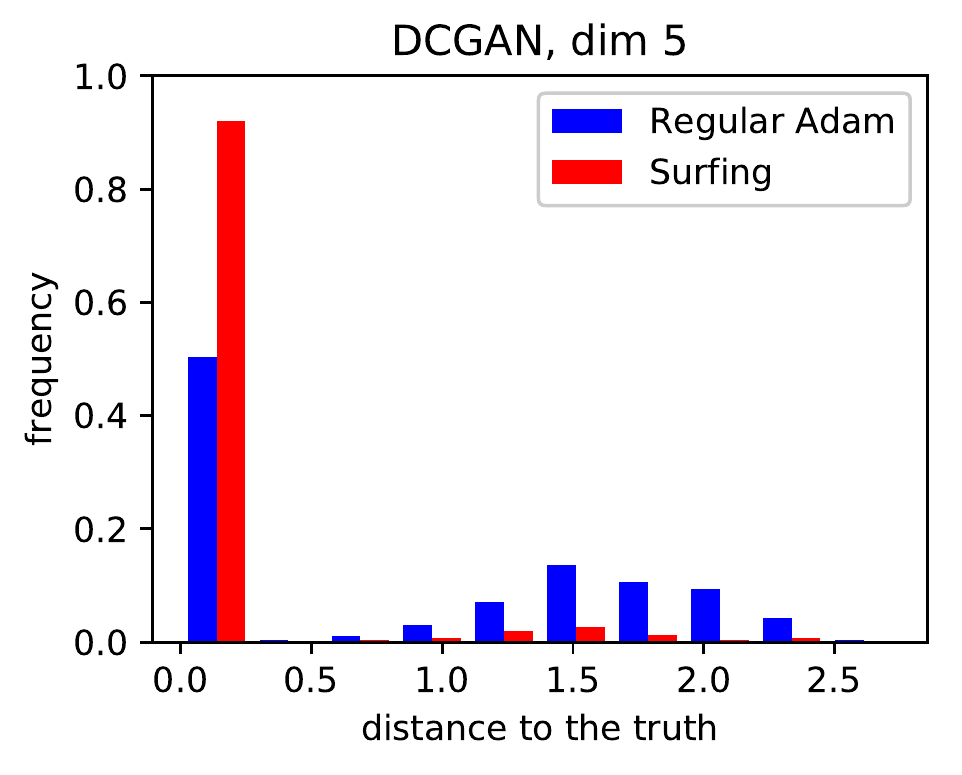} &
  \includegraphics[width=.3\textwidth]{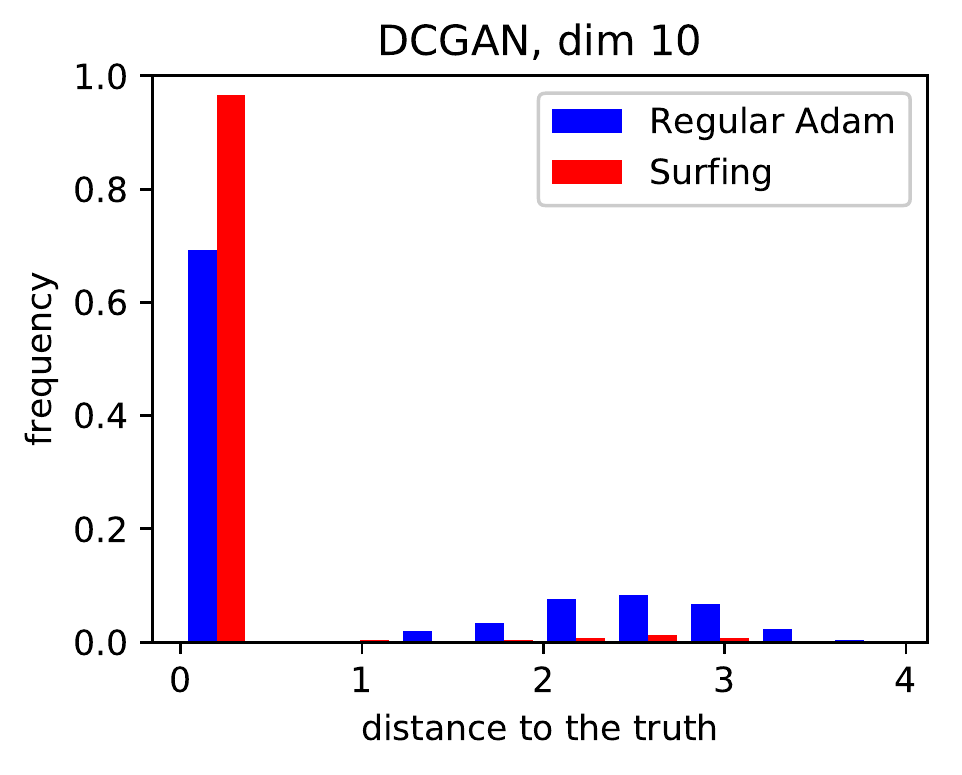} &
  \includegraphics[width=.3\textwidth]{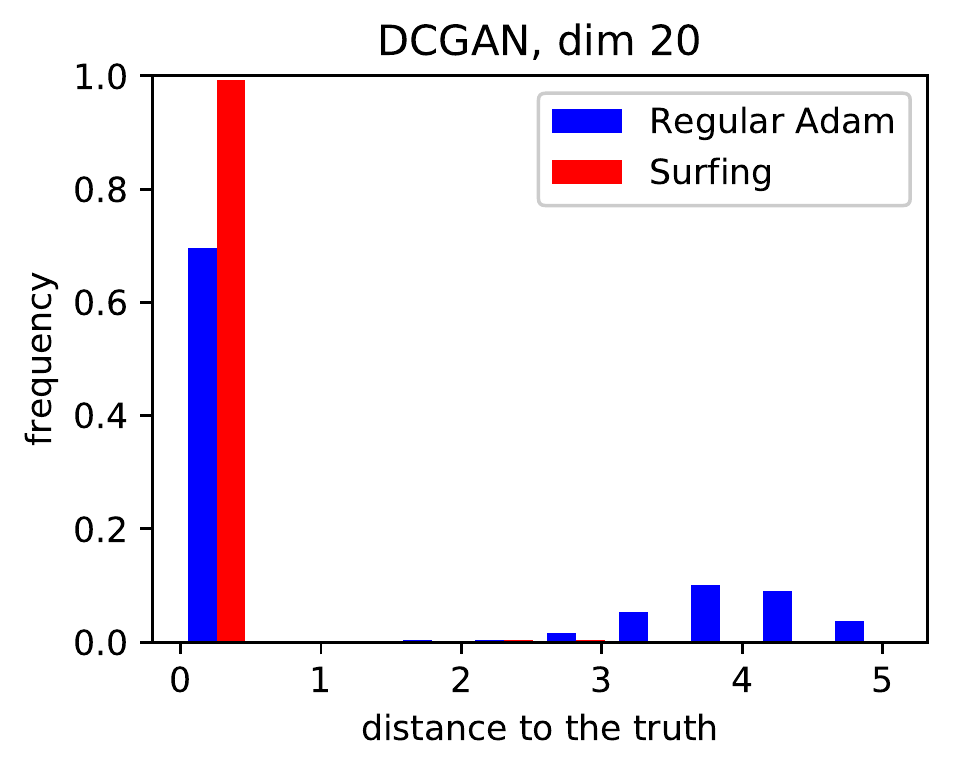} \\[-10pt]
  \end{tabular}
  \end{center}
  \caption{Distribution of distance between solution $\hat x_T$ and the truth $x_*$ for DCGAN trained models, comparing surfing (red) to regular gradient descent (blue) over the final network. Both procedures use Adam in their gradient descent computations. The results indicate that direct descent often succeeds, but can also converge to a point that is far from the optimum. By moving along the optimum of the evolving surface, surfing is able to move closer to the optimum in these cases.}
  \vskip18pt
    \label{fig:dist-WGAN}
  \end{figure}
}

We present experiments to illustrate the performance of surfing over a sequence of networks during training compared with gradient descent over the final trained network. We mainly use the Fashion-MNIST dataset\footnote{https://github.com/zalandoresearch/fashion-mnist} to carry out the simulations, which is similar
to MNIST in many characteristics, but is more difficult to train. We build multiple generative models, trained using VAE \citep{kingma2013auto}, DCGAN \citep{radford2015unsupervised}, WGAN \citep{arjovsky2017wasserstein} and WGAN-GP \citep{gulrajani2017improved}. The structure of the generator/decoder networks that we use are the same as those reported by  \cite{chen2016infogan}; they include two fully connected layers and two transposed convolution layers with batch normalization after each layer \citep{ioffe2015batch}. We use the simple surfing algorithm in these experiments, rather than the projected-gradient algorithm proposed for theoretical analysis. Note also that the network architectures do not precisely match the expansive relu networks used in our analysis. Instead, we experiment with architectures and training procedures that are meant to better reflect the current state of the art.

We first consider the problem of minimizing the objective $f(x) =
\frac{1}{2}\|G(x) - G(x_*)\|^2$ and recovering the image generated from a
trained network $G(x) = G_{\theta_T}(x)$ with input $x_*$. We run surfing by
taking a sequence of parameters $\theta_0, \theta_1,...,\theta_T$, where
$\theta_0$ are the initial random parameters and the intermediate $\theta_t$'s
are taken every 40 training steps. In order to improve convergence speed, we use
Adam \citep{kingma2014adam} to carry out gradient descent in $x$ during each
surfing step. We also use Adam when optimizing over $x$ in only  the final network. For each network training condition we apply surfing and regular Adam for 300 trials, where in each trial a randomly generated $x_*$ and initial point $x_{init}$ are chosen uniformly from the hypercube $[-1,1]^k$. Table \ref{table:prop} shows the percentage of trials where the solutions $\hat x_T$ satisfy $\|\hat x_T - x_*\| < 0.01$ for different models, over three different input dimensions $k$. We also provide the distributions of $\|\hat x_T - x_*\|$ under each setting. Figure \ref{fig:dist-WGAN} shows the results for DCGAN.

We next consider the compressed sensing problem with objective $f(x) = \frac{1}{2}\|AG(x) - AG(x_*)\|^2$ where $A\in\R^{m\times n}$ is the Gaussian measurement matrix. We carry out 200 trials for each choice of number of measurements $m$. The parameters $\theta_t$ for surfing are taken every 100 training steps. As before, we record the proportion of the solutions that are close to the truth $x_*$ according to $\|\hat x_T - x_*\| < 0.01$. Figure \ref{fig:prop} shows the results for DCGAN and WGAN trained networks with input dimension $k=20$.

Lastly, we consider the objective $f(x) = \frac{1}{2}\|AG(x) - Ay\|^2$, where $y$ is a real image from the hold-out test data.
This can be thought of as a rate-distortion setting, where the error varies as a function of the number of measurements used.
We carry out the same experiments as before and compute the average per-pixel reconstruction error $\sqrt{\frac{1}{n}\|G(\hat x_T) - y\|^2}$ as in \cite{bora2017compressed}. Figure \ref{fig:reconstr} shows the distributions of the reconstruction error as the number of measurements $m$ varies.

\begin{figure}[ht!]
  \begin{center}
  \begin{tabular}{cc}
  \includegraphics[width=.40\textwidth]{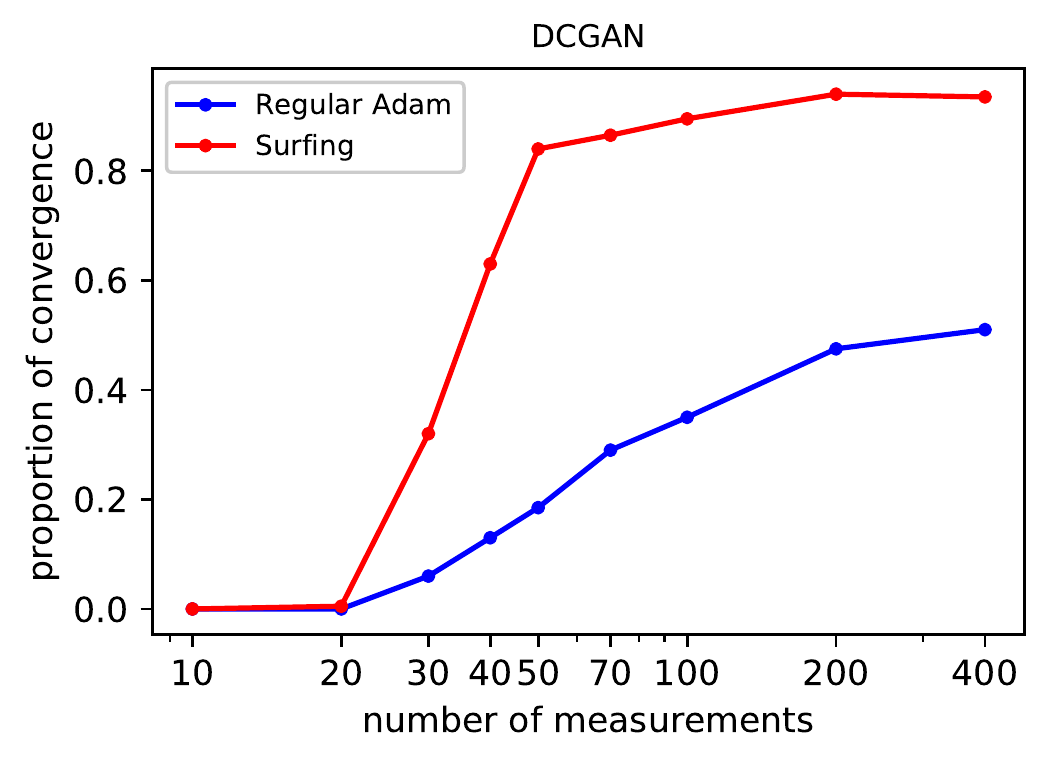} &
  \includegraphics[width=.40\textwidth]{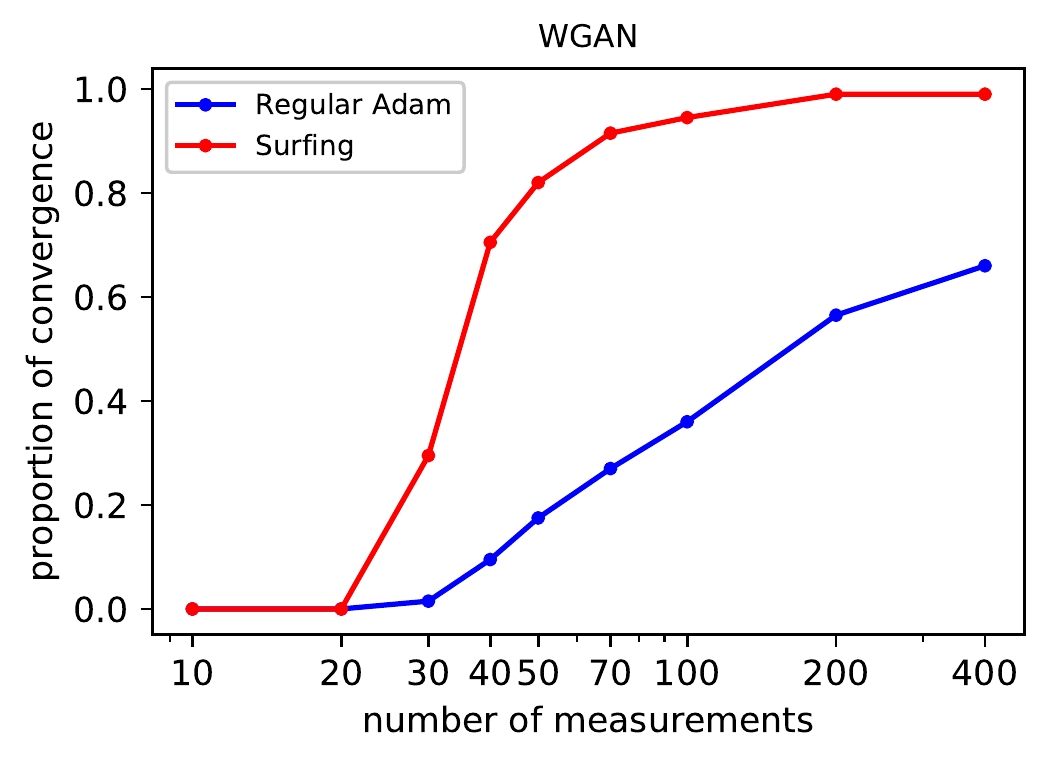} \\[-10pt]
  \end{tabular}
  \end{center}
  \caption{Compressed sensing setting for exact recovery. As a function of the number
  of random measurements $m$, the lines show the proportion of times surfing (red) and
  regular gradient descent with Adam (blue) are able to recover the true signal $y = G(x)$, using DCGAN and WGAN.}
  \label{fig:prop}
\vskip10pt
\begin{center}
\begin{tabular}{cc}
\includegraphics[width=.40\textwidth]{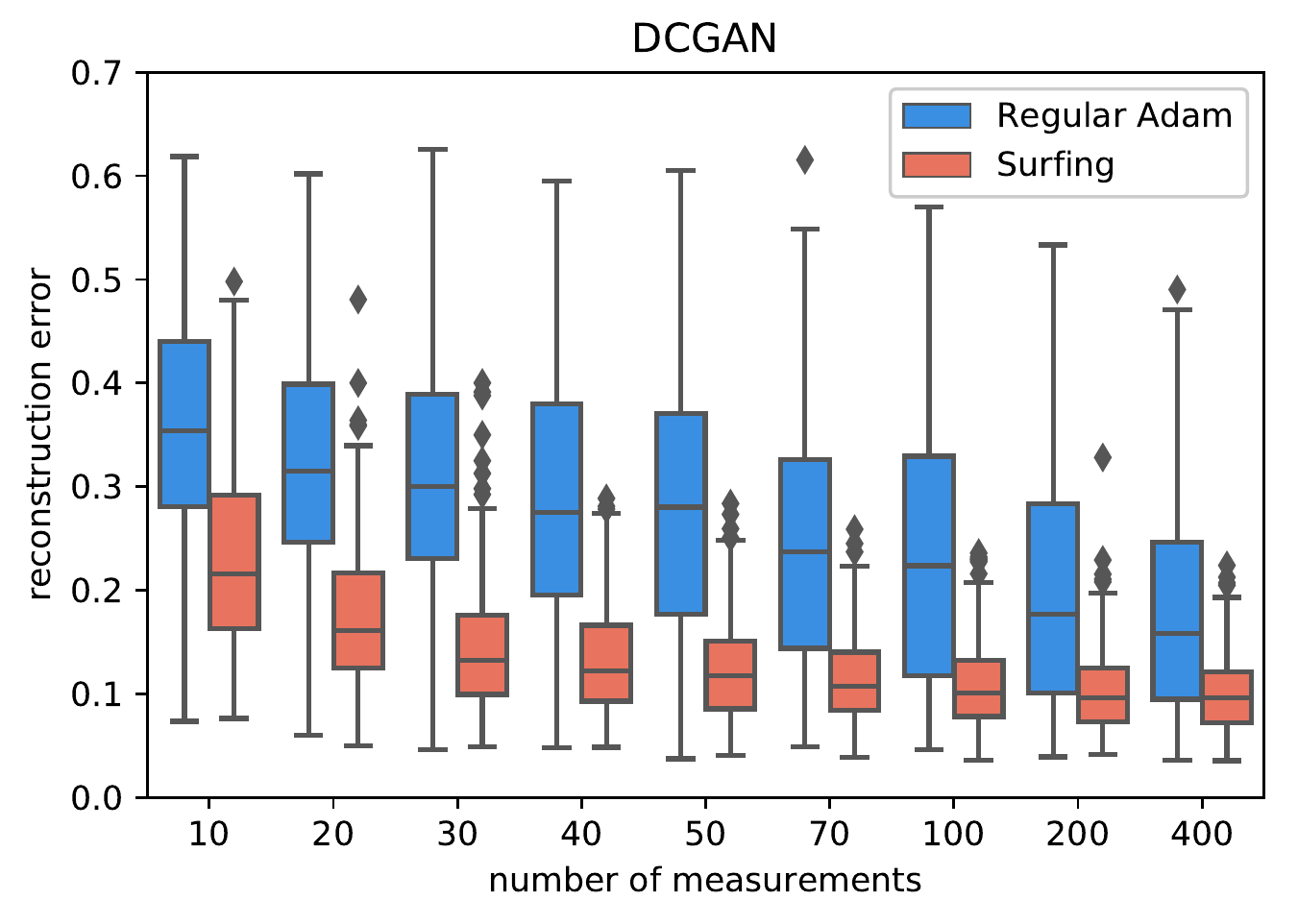} &
\includegraphics[width=.40\textwidth]{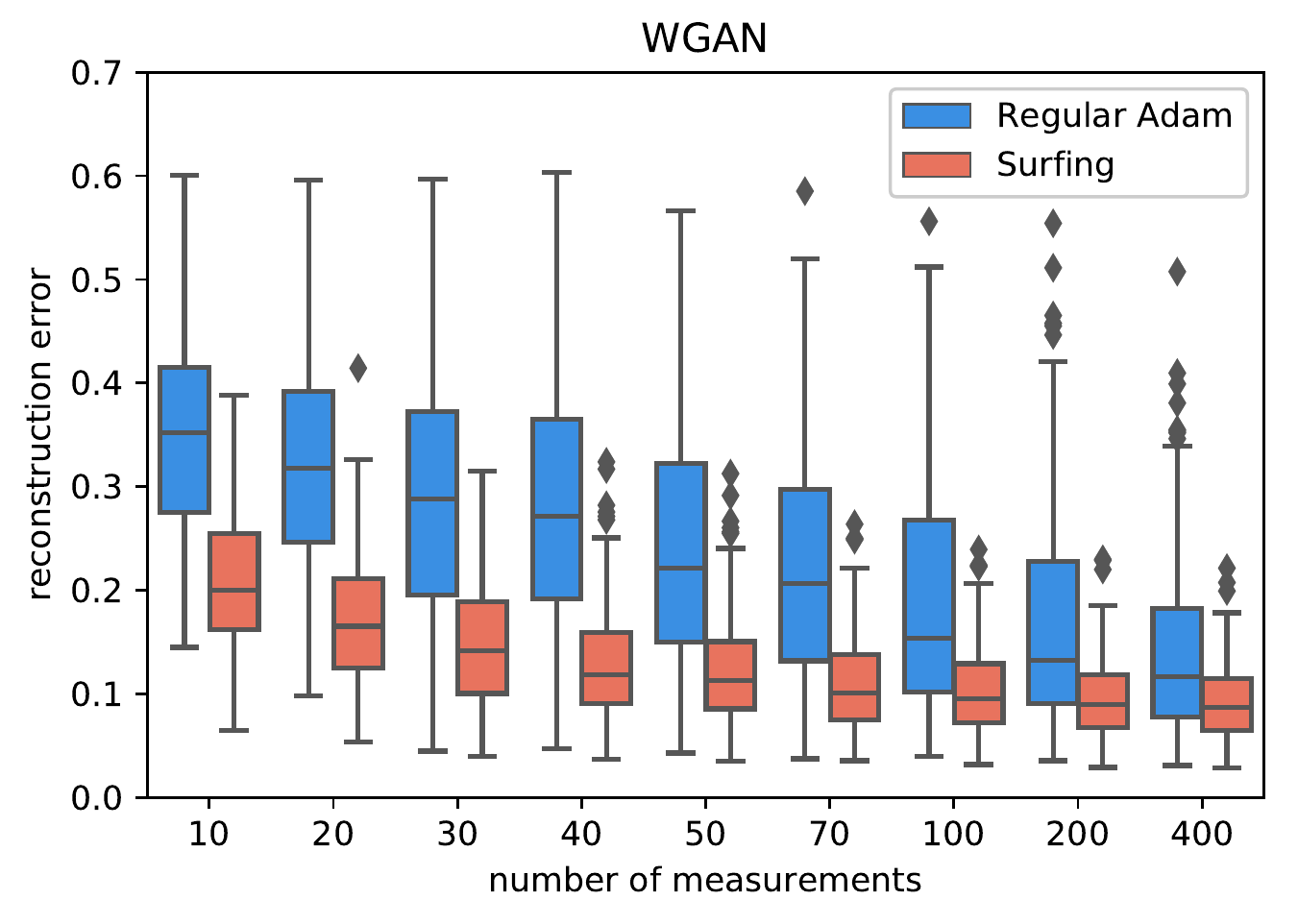} \\[-10pt]
\end{tabular}
\end{center}
\caption{Compressed sensing setting for approximation, or rate-distortion.
As a function of the number
of random measurements $m$, the box plots summarize the distribution of the per-pixel reconstruction errors for DCGAN and WGAN trained models, using surfing (red) and regular gradient descent with Adam (blue). }
\label{fig:reconstr}
\end{figure}


Figure~\ref{fig:more} shows additional plots for experiments comparing surfing over a sequence of networks during training to gradient descent over the final trained network. As described above, 
we consider the problem of minimizing the objective $f(x) = \frac{1}{2}\|G(x) - G(x_*)\|^2$, that is, recovering the image generated from a trained network $G(x) = G_{\theta_T}(x)$ with input $x_*$. We run surfing by taking a sequence of parameters $\theta_0, \theta_1,...,\theta_T$, where $\theta_0$ are the initial random parameters and the intermediate $\theta_t$'s are taken every 40 training steps. In order to improve convergence speed we use Adam \citep{kingma2014adam} to carry out gradient descent in each step in surfing. We also use Adam when optimizing over the just the final network. We apply surfing and regular Adam for 300 trials, where in each trial a randomly generated $x_*$ and initial point $x_{init}$ is chosen. Figure \ref{fig:more} shows
the distribution of the distance between the computed solution $\hat x_T$ and the truth $x_*$ for VAE, WGAN and WGAN-GP,
using surfing (red) and regular gradient descent with Adam (blue), over three different input dimensions $k$.

\begin{figure}[ht!]
\begin{center}
\begin{tabular}{ccc}
  \\[20pt]
\includegraphics[width=.3\textwidth]{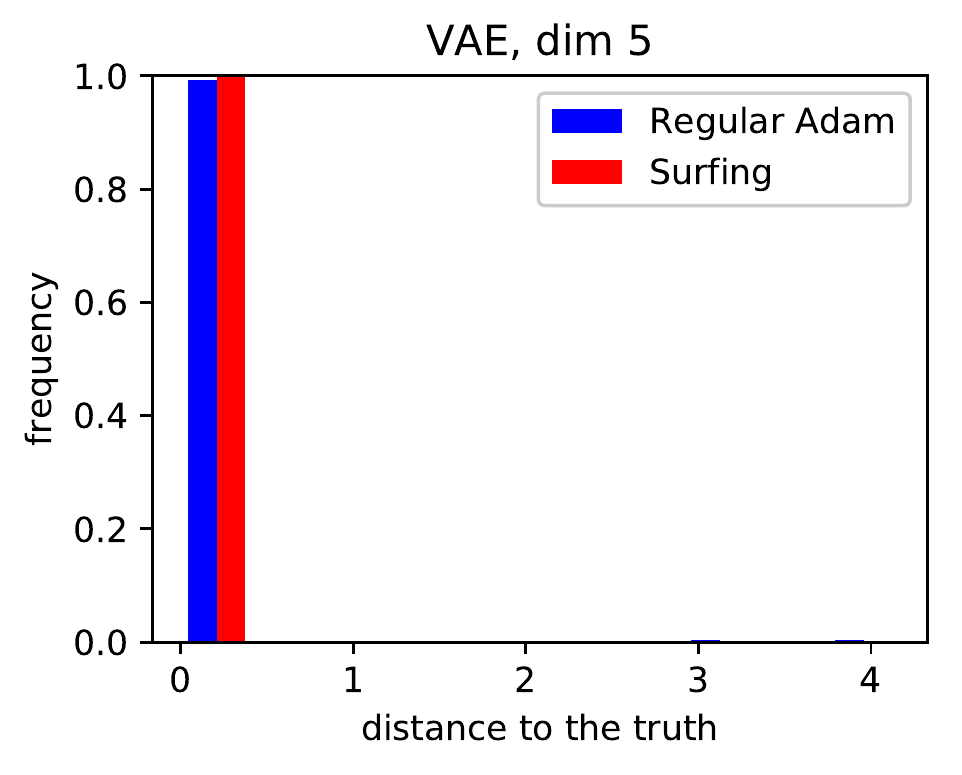} &
\includegraphics[width=.3\textwidth]{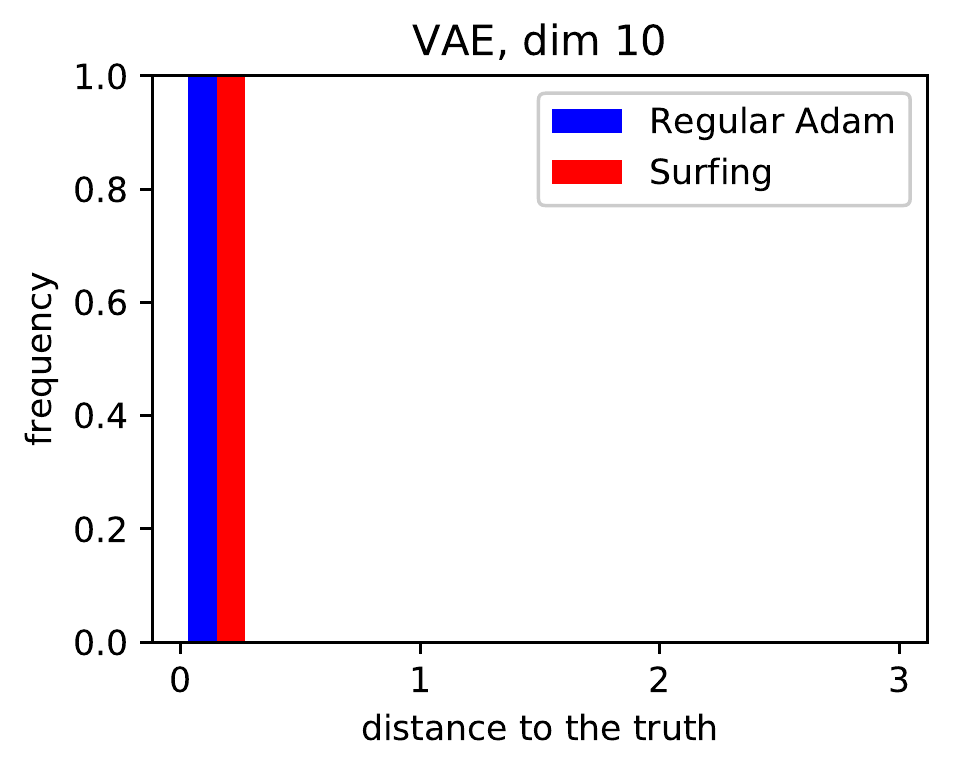} &
\includegraphics[width=.3\textwidth]{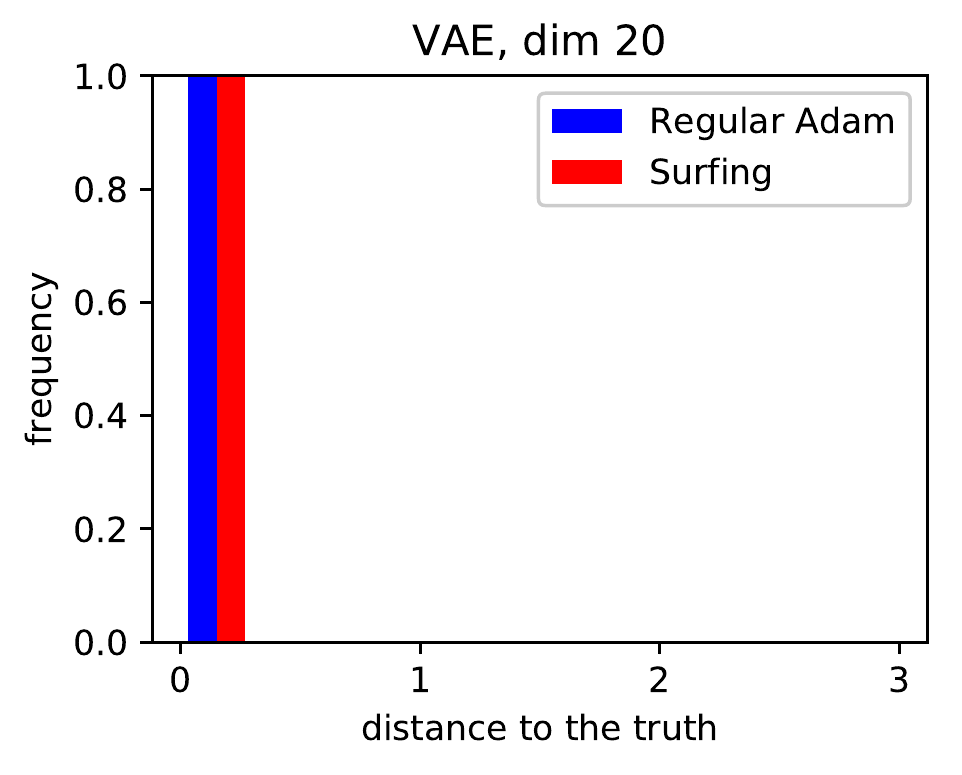} \\
\includegraphics[width=.3\textwidth]{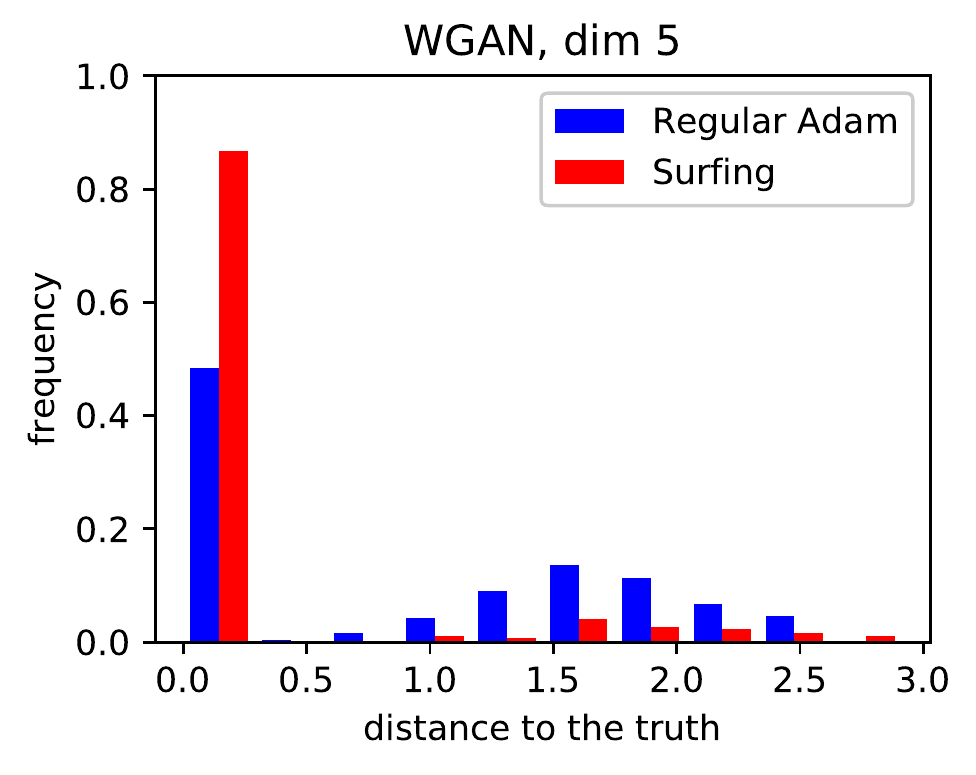} &
\includegraphics[width=.3\textwidth]{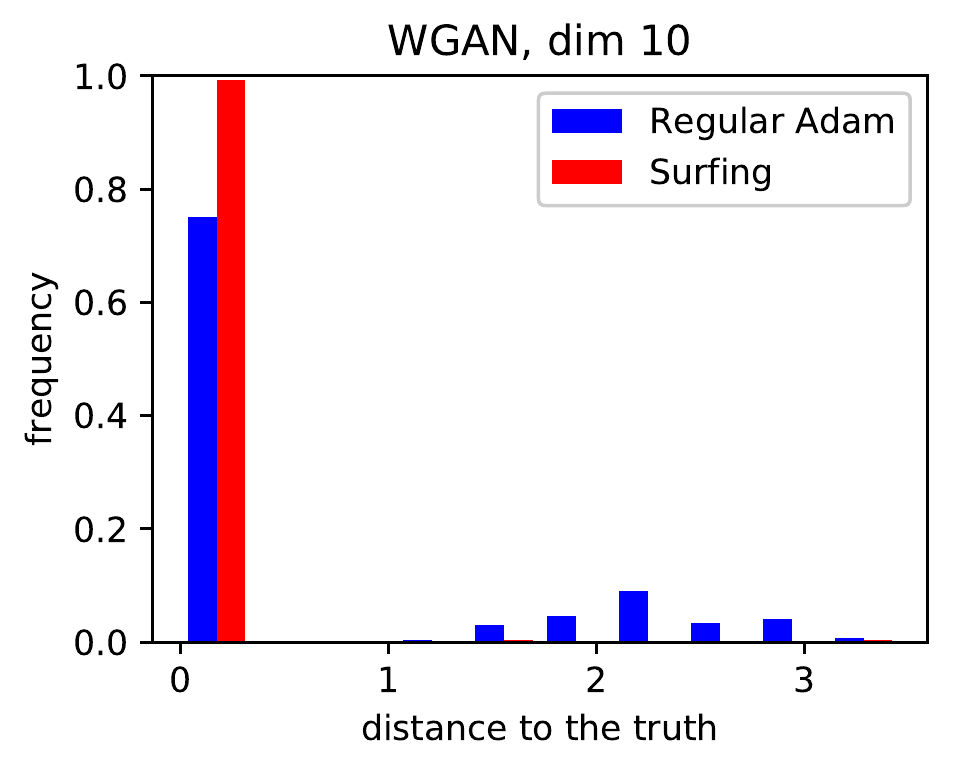} &
\includegraphics[width=.3\textwidth]{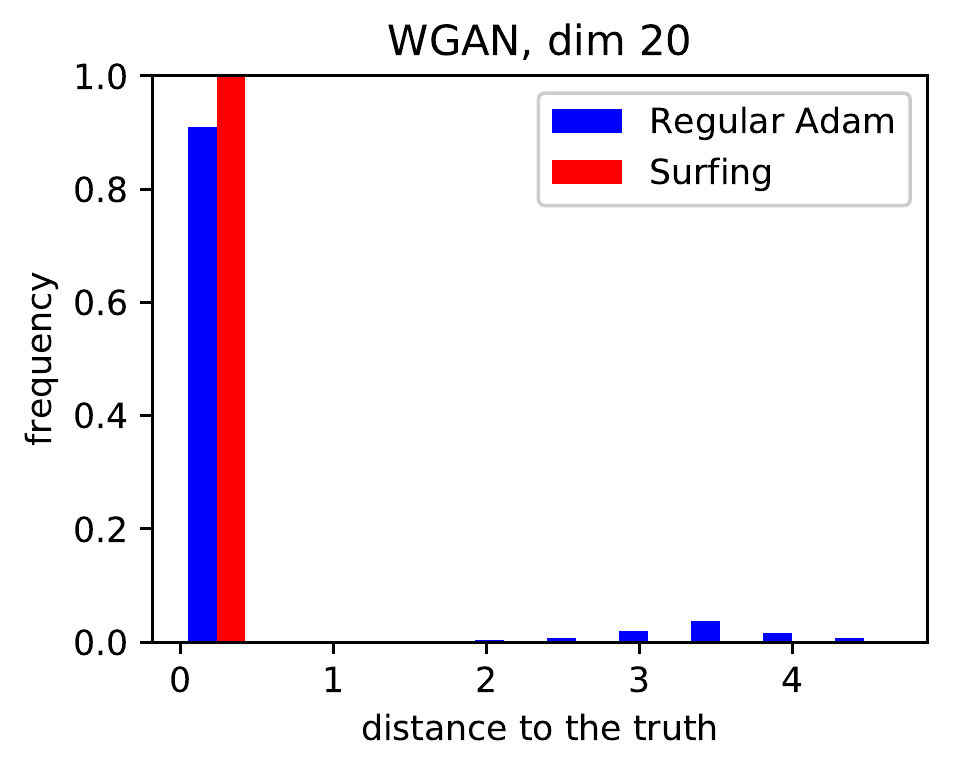}\\
\includegraphics[width=.3\textwidth]{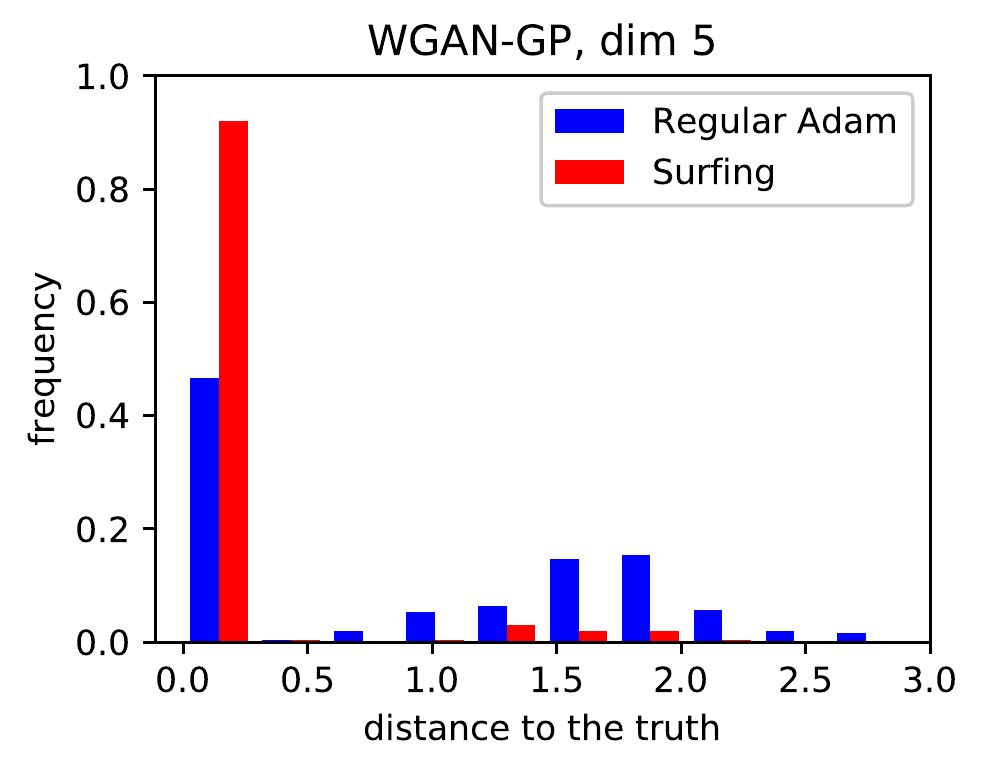} &
\includegraphics[width=.3\textwidth]{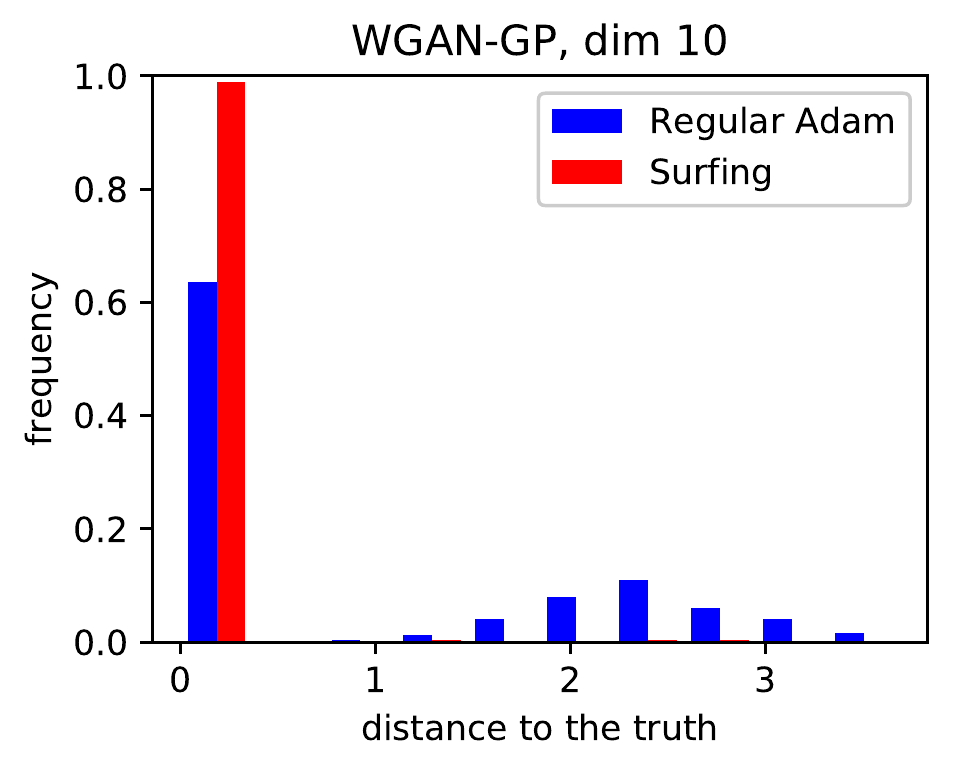} &
\includegraphics[width=.3\textwidth]{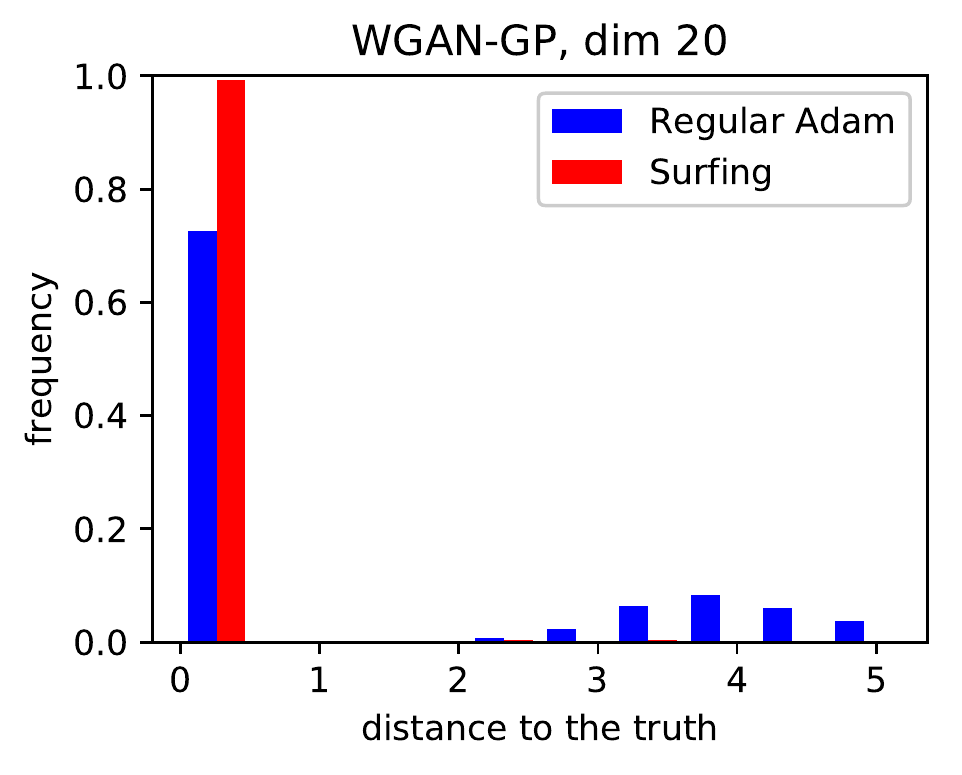}
\end{tabular}
\end{center}
\caption{Distribution of the distance between solution $\hat x_T$ and the truth $x_*$ for VAE, WGAN and WGAN-GP,
using surfing (red) and regular gradient descent with Adam (blue) over three different input dimensions $k$.}
\label{fig:more}
\end{figure}

%% file: proofs.tex

\def\sign{\text{sign}}
\def\diag{\text{diag}}
\def\I{\mathbbm{1}}


\section{Proof of Theorem \ref{thm:init}}\label{sec:proof}

We denote $[n] = \{1,2,...,n\}$, $\Pi_{i=1}^d W_i=W_1W_2\ldots W_d$, and
$\Pi_{i=d}^1 W_i=W_dW_{d-1}\cdots W_1$. $\|x\|$ and $\|A\|$ are the Euclidean
vector norm and matrix operator norm. $C,C',c,c'>0$
denote $d$-dependent constants that may change from instance to instance.

We adapt ideas of \cite{hand2017global}.
Denote for simplicity $G(x)=G(x,\theta_0)$ and $f(x)=f_0(x)$.
Define
\begin{equation*}
  W_{i,+,v}=\diag(W_iv+b_i>0)W_i, \qquad b_{i,+,v}=\diag(W_iv+b_i>0)b_i
\end{equation*}
where $\text{diag}(w>0)$ denotes a diagonal matrix with $j$th diagonal element
$\I\{w_j>0\}$. Then
\begin{equation*}
  \sigma(W_iv+b_i)=W_{i,+,v}v+b_{i,+,v}.
\end{equation*}
The analysis of \cite{hand2017global} shows that the matrices
\begin{equation*}
  \tilde{W}_{i,+,v} \equiv \begin{pmatrix} W_{i,+,v} & b_{i,+,v} \end{pmatrix}
\in \R^{n_i \times (n_{i-1}+1)}
\end{equation*}
satisfy a certain Weight Distribution Condition (WDC), yielding a deterministic
approximation for $\tilde{W}_{i,+,v}^\top \tilde{W}_{i,+,v'}$ and any $v,v' \in
\R^{n_{i-1}}$. We will use the following consequence of this condition.

\begin{lemma}\label{lemma:WDC}
Under the conditions of Theorem \ref{thm:init},
with probability at least $1-C\sum_{i=1}^d n_ie^{-c\eps^2 n_{i-1}}$,
the following hold for every $i \in [d]$ and $v,v' \in \R^{n_{i-1}}$:
\begin{enumerate}
\item[(a)] $\|W_{i,+,v}\| \leq 2$ and $\|b_{i,+,v}\| \leq 2$.
\item[(b)] $\|W_{i,+,v}^\top W_{i,+,v'}-\frac{1}{2}I\| \leq \eps+\theta/\pi$,
where $\theta$ is the angle formed by $v$ and $v'$.
\item[(c)] $\|W_{i,+,v}^\top b_{i,+,v}\| \leq \eps$.
\end{enumerate}
\end{lemma}
\begin{proof}
For (a), note that $\|W_i\| \leq 2$ and
$\|b_i\| \leq 2$ with probability $1-e^{-cn_i}$, by a standard
$\chi^2$ tail-bound and operator norm bound for a Gaussian matrix. On the event
that these hold, the bounds
hold also for $W_{i,+,v}$ and $b_{i,+,v}$ and every $v \in \R^{n_{i-1}}$.

For (b) and (c), by \cite[Lemma 11]{hand2017global}, with probability
$1-8n_ie^{-c\eps^2 n_{i-1}}$ the matrix
$\tilde{W}_{i,+,v}$ satisfies WDC
with constant $\eps$ for every $v$.
(The dependence of the constants $c,\gamma$ in
\cite[Lemma 11]{hand2017global} are given by $c \gtrsim \eps^{-2}\log \eps^{-1}$
and $\gamma \lesssim \eps^2$ as indicated in the proof. This condition for $c$
matches the growth rate of $n_i$ specified in our Theorem \ref{thm:init}.)
From the form of $Q$ in \cite[Definition 2]{hand2017global}, the WDC implies
\begin{equation*}
  \left\|\tilde{W}_{i,+,v}^\top \tilde{W}_{i,+,v'}-\frac{1}{2}I\right\|
\leq \eps+\tilde{\theta}/\pi
\end{equation*}
where $\tilde{\theta}$ is the angle between $(v,1)$ and $(v',1)$. Noting that
$\tilde{\theta} \leq \theta$ and recalling the definition of
$\tilde{W}_{i,+,v}$, we get (b) and (c).
\end{proof}

For $x \in \R^k$, let $x_0=x$ and let $x_i=\sigma(W_i \ldots
\sigma(W_1x+b_1) \ldots + b_i)$ be the output of the $i$th layer. Denote
\begin{equation*}
W_{i,x}=W_{i,+,x_{i-1}}, \qquad b_{i,x}=b_{i,+,x_{i-1}}.
\end{equation*}
Then also $x_i=W_{i,x}x_{i-1}+b_{i,x}$.

\begin{lemma}\label{lemma:subspacecount}
Under the conditions of Theorem \ref{thm:init}, with probability 1, the total
number of distinct possible tuples $(W_{1,x},b_{1,x},\ldots,W_{d,x},b_{d,x})$
satisfies
\begin{equation*}
|\{(W_{1,x},b_{1,x},\ldots,W_{d,x},b_{d,x}):x \in \R^k\}|
\leq 10^{d^2}(n_1\ldots n_d)^{d(k+1)}.
\end{equation*}
\end{lemma}
\begin{proof}
Let $S=\R^{k+1}$, which contains $(x,1)$. Then
the result of \cite[Lemma 15]{hand2017global} applied to the vector space
$S$ and to $\tilde{W}_{1,x}=(W_{1,x} \;\; b_{1,x})$ yields
\begin{equation*}
|\{(W_{1,x},b_{1,x}:x \in \R^k)\}| \leq 10n_1^{k+1}.
\end{equation*}
Each distinct $(W_{1,x},b_{1,x})$ defines an affine linear space of
dimension $k$ which contains the first layer output $x_1$, and hence a subspace
$S$ of dimension $k+1$ which contains $(x_1,1)$. Applying
\cite[Lemma 15]{hand2017global} to each such $S$ and $\tilde{W}_{2,x}$ yields
\begin{equation*}
|\{(W_{2,x},b_{2,x}:x \in \R^k)\}| \leq 10n_1^{k+1} \cdot 10n_2^{k+1}.
\end{equation*}
Proceeding inductively,
\begin{equation*}
|\{(W_{i,x},b_{i,x}:x \in \R^k)\}| \leq 10^i (n_1\ldots n_i)^{k+1},
\end{equation*}
which is analogous to \cite[Lemma 16]{hand2017global} in our setting with
biases $b_1,\ldots,b_d$.
The result follows from taking the product over $i=1,\ldots,d$.
\end{proof}

\begin{lemma}\label{lemma:RIP}
Let $A \in \R^{m \times n}$ have i.i.d.\ $\N(0,1/m)$ entries. Fix $\eps>0$, let
$k<n$, and let
$V=\bigcup_{i=1}^M V_i$ and $W=\bigcup_{j=1}^N W_j$ where $V_i$ and $W_j$ are
subspaces of dimension at most $k$. Then with probability at least
$1-MN(c/\eps)^{2k}e^{-c'\eps m}$, for all $x \in V$ and $y \in W$ we have
\begin{equation*}
|x^\top A^\top Ay-x^\top y| \leq \eps\|x\|\|y\|.
\end{equation*}
\end{lemma}
\begin{proof}
See \cite[Lemma 14]{hand2017global}.
\end{proof}

Using these results, we analyze the gradient and critical points of $f(x)$.
Note that with the above definitions,
\begin{align*}
G(x)&=V(W_{d,x} \ldots(W_{1,x}x+b_{1,x}) \ldots+b_{d,x})\\
&=V\left(\prod_{i=d}^1 W_{i,x}\right)x
+V\sum_{j=1}^d \left(\prod_{i=d}^{j+1} W_{i,x}\right)b_{j,x}.
\end{align*}
The function $G(x)$ is piecewise linear in $x$, so $f(x)$ is piecewise
quadratic.
If $f(x)$ is differentiable at $x$, then the gradient of $f$ can be written as
\begin{align*}
\nabla f(x)&=\left(\prod_{i=1}^d W_{i,x}^\top\right)V^\top A^\top
\left(AV\left(\prod_{i=d}^1 W_{i,x}\right)x
+AV\sum_{j=1}^d \left(\prod_{i=d}^{j+1} W_{i,x}\right)b_{j,x}
-Ay\right).
\end{align*}

\begin{lemma}
\label{lma:1}
Define
\begin{equation*}
g_x=2^{-d}x-\left(\prod_{i=1}^d W_{i,x}^\top \right) V^\top y
\end{equation*}
Under the conditions of Theorem \ref{thm:init}, we have with probability
$1-C(e^{-c\eps m}+e^{-c\eps n}+\sum_i n_ie^{-c\eps^2 n_{i-1}})$ that
at every $x \in \R^k$ where $f$ is differentiable,
\begin{equation*}
\|\nabla f(x) - g_x\| \leq C'\eps (1+\|x\|+\|y\|)
\end{equation*}
\end{lemma}
\begin{proof}
By Lemma \ref{lemma:subspacecount}, for fixed
$\theta=(V,W_1,b_1,\ldots,W_d,b_d)$,
the range $\{V\prod_{i=d}^1 W_{i,x}x':x,x' \in \R^k\}$
belongs to a union of at most $C(n_1\ldots n_d)^{d(k+1)}$ subspaces of dimension
$k$. For some $C',c>0$, under the condition
$m \geq C'k(\eps^{-1}\log \eps^{-1})\log(n_1\ldots n_d)$, we have
\begin{equation*}C^2(n_1\ldots n_d)^{2d(k+1)}(c/\eps)^{2k}e^{-c'\eps
m}\leq e^{-c\eps m}.
\end{equation*}
Then for $A \in \R^{m \times n}$ with i.i.d.\ $\N(0,1/m)$ entries,
applying Lemma \ref{lemma:RIP} conditional on $\theta$,
and then \ref{lemma:WDC}(a) to bound $\|W_{i,x}\|$ and $\|V\|$, we get
\begin{equation*}
\left\|\left(\prod_{i=1}^d W_{i,x}^\top\right)V^\top (A^\top
A-I)V\left(\prod_{i=d}^1 W_{i,x}\right)x\right\|
\leq C\eps\|x\|.
\end{equation*}
For $A=I$, this bound is trivial. The given conditions imply also
\begin{equation*}
n \geq n_d \geq C'k(\eps^{-1}\log \eps^{-1})\log
(n_1\ldots n_d),
\end{equation*}
so applying the same argument with $V$ in place of $A$ yields
\begin{equation*}
\left\|\left(\prod_{i=1}^d W_{i,x}^\top\right)(V^\top
V-I)\left(\prod_{i=d}^1 W_{i,x}\right)x\right\|
\leq C\eps\|x\|.
\end{equation*}
Next, applying Lemma \ref{lemma:WDC}(a--b) yields, for each
$j=d,d-1,\ldots,2,1$,
\begin{equation*}
\left\|\left(\prod_{i=1}^{j-1} W_{i,x}^\top\right)
(W_{j,x}^\top W_{j,x}-I/2)\left(\prod_{i=j-1}^1 W_{i,x}\right)x\right\|
\leq C\eps\|x\|.
\end{equation*}
Combining these results, we get for the first term
of $\nabla f(x)$ that
\begin{equation}\label{eq:nablaf1}
\left\|\left(\prod_{i=1}^d W_{i,x}^\top\right)V^\top A^\top
AV\left(\prod_{i=d}^1 W_{i,x}\right)x-2^{-d}x\right\| \leq C\eps\|x\|.
\end{equation}
This holds with probability at least $1-e^{-c\eps m}-e^{-c\eps n}-C\sum_i
n_ie^{-cn_{i-1}}$.

The second term is controlled similarly: Lemma \ref{lemma:subspacecount}
implies that for fixed parameters $\theta$,
the set $\{V\prod_{i=d}^{j+1} W_{i,x}b_{j,x}:x \in \R^k,j \in [d]\}$
is comprised of
at most one of $C(n_1\ldots n_d)^{d(k+1)}$ distinct vectors (which belong to
subspaces of dimension 1.)
Then applying Lemma \ref{lemma:RIP} twice to $A$ and $V$ as above, and
using also $\|b_{j,x}\| \leq 2$ from Lemma \ref{lemma:WDC}(a),
\begin{equation*}
\left\|\left(\prod_{i=1}^d W_{i,x}^\top\right)(V^\top A^\top
AV-I)\left(\prod_{i=d}^{j+1} W_{i,x}\right)b_{j,x}\right\|
\leq C\eps.
\end{equation*}
Applying Lemma \ref{lemma:WDC}(a--b) iteratively as above, we get
\begin{equation*}
\left\|\left(\prod_{i=1}^j W_{i,x}^\top\right)
\left[\left(\prod_{i=j+1}^d W_{i,x}^\top\right)
\left(\prod_{i=d}^{j+1} W_{i,x}\right)-2^{-(d-j)}I\right]b_{j,x}\right\|
\leq C\eps.
\end{equation*}
Finally, Lemma \ref{lemma:WDC}(a) and (c) yield
\begin{equation*}
\left\|\left(\prod_{i=1}^j W_{i,x}^\top\right)b_{j,x}\right\| \leq C\eps.
\end{equation*}
Combining these, we have for the second term of $\nabla f(x)$ that
\begin{equation}\label{eq:nablaf2}
\left\|\sum_{j=1}^d \left(\prod_{i=1}^d W_{i,x}^\top\right)V^\top A^\top
AV\left(\prod_{i=d}^{j+1} W_{i,x}\right)b_{j,x}\right\| \leq C\eps
\end{equation}
also with probability $1-e^{-c\eps m}-e^{-c\eps n}-C\sum_i n_ie^{-c\eps^2
n_{i-1}}$.

Finally, for the last term of $\nabla f(x)$, if $A \neq I$ then we may apply
Lemma \ref{lemma:RIP} again to get
\begin{equation}\label{eq:nablaf3}
\left\|\left(\prod_{i=1}^d W_{i,x}^\top \right) V^\top (A^\top A-I)y
\right\| \leq C\eps\|y\|
\end{equation}
with probability $1-e^{-c\eps m}$. Combining (\ref{eq:nablaf1}),
(\ref{eq:nablaf2}), and (\ref{eq:nablaf3}) concludes the proof.
\end{proof}

We now bound the second term of $g_x$.
\begin{lemma}
\label{lma:2}
Under the conditions of Theorem \ref{thm:init}, with probability
$1-Cn_de^{-c\eps^4n_{d-1}}$, for every $v \in \R^{n_{d-1}}$
\begin{equation*}
\left\|W_{d,+,v}^\top V^\top y \right\| \leq C\eps\|y\|.
\end{equation*}
\end{lemma}
\begin{proof}
Note that $V^\top y \in \R^{n_d}$ has i.i.d.\ $\mathcal{N}(0,\|y\|^2/n)$ entries.
Then conditional on $W_d$, for each fixed $v \in \R^{n_{d-1}}$,
\begin{equation*}
u(v) \equiv W_{d,+,v}^\top V^\top y \sim \N(0,\Sigma)
\end{equation*}
where
\begin{equation*}
\Sigma=(\|y\|^2/n) \cdot W_{d,+,v}^\top W_{d,+,v} \in \R^{n_{d-1} \times
n_{d-1}}.
\end{equation*}
On the event that Lemma \ref{lemma:WDC}(b) holds, we have $\|\Sigma\| \leq
\|y\|^2/n$ and hence $\|u(v)\|^2 \leq tn_{d-1}\|y\|^2/n$ with probability
$1-e^{cn_{d-1}t}$ for large $t$, by a $\chi^2$ tail-bound. Noting that
$n \geq n_d \gg \eps^{-2}n_{d-1}$ and applying this bound
for $t=\eps^2n/n_{d-1}$,
we get $\|u(v)\| \leq \eps\|y\|$ with probability $1-e^{-c\eps^2 n}$.

We use a covering net argument to take a union bound over $v$:
Let $N$ be an $\eps^2$-net of the $n_{d-1}$-sphere, of cardinality $|N| \leq
(3/\eps^2)^{n_{d-1}}$.
The above holds uniformly over $v \in N$ with probability $1-e^{c'\eps^2 n}$,
because $n \geq n_d \gg n_{d-1} \cdot \eps^{-2}\log \eps^{-1}$.
For any $v'$ on the sphere and $v \in N$ with $\|v-v'\|<\eps^2$, the angle
$\theta$ between $v$ and $v'$ is at most $C\eps^2$. We have
\begin{equation*}
\|u(v)-u(v')\|
\leq \left\|W_{d,+,v}^\top-W_{d,+,v'}^\top\right\| \cdot \|V^\top y\|.
\end{equation*}
Suppose now that Lemma \ref{lemma:WDC}(b) holds for $W_d$
with the constant $\eps^2$: This occurs with probability $1-8n_de^{-c\eps^4
n_{d-1}}$. Approximating each of the four terms in
\begin{equation*}
\left(W_{d,+,v}^\top- W_{d,+,v'}^\top\right)
\left(W_{d,+,v} - W_{d,+,v'}\right)
\end{equation*}
by $I/2$ on this event, we get
\begin{equation*}
\left\|W_{d,+,v}^\top-W_{d,+,v'}^\top\right\|^2
=\left\|\left(W_{d,+,v}^\top- W_{d,+,v'}^\top\right)
\left(W_{d,+,v} - W_{d,+,v'}\right)\right\|
\leq C'(\eps^2+\theta) \leq C\eps^2.
\end{equation*}
Thus on this event,
$\|u(v)-u(v')\| \leq C\eps \|V^\top y\|$. By a $\chi^2$ tail-bound, with
probability $1-e^{-cn_d}$ we have $\|V^\top y\|^2 \leq 2\|y\|^2n_d/n
\leq 2\|y\|^2$ and hence $\|u(v)-u(v')\| \leq C\eps \|y\|$.
\end{proof}

\begin{proof}[Proof of Theorem \ref{thm:init}]
Combining Lemmas \ref{lma:1}, \ref{lma:2}, and \ref{lemma:WDC}(a),
with the stated probability,
\begin{equation*}
\|\nabla f(x)-2^{-d}x\| \leq C\eps(1+\|x\|+\|y\|)
\end{equation*}
for every $x \in \R^k$. Since $G$ is piecewise
linear, the directional derivative $D_v f(x)$ always exists at any $x \in \R^k$
for any unit vector $v \in \R^k$, even for $x$ where $f$ is non-differentiable.
Set $\tilde{x}=x/\|x\|$. For any fixed $x$, there exists a sequence
$\{x_n\}$ which converges to $x$ and where $f$ is differentiable, such that
\begin{equation*}
D_{-\tilde{x}} f(x) = \lim_{n\to \infty} -\tilde{x}^\top \nabla f(x_n)
\end{equation*}
Since
\begin{equation*}
-\tilde{x}^\top \nabla f(x_n) = -2^{-d} \tilde{x}^\top x_n
+\tilde{x}^\top (2^{-d}{x_n}-\nabla f(x_n))
\leq -2^{-d} \tilde{x}^\top x_n+C\eps(1+\|x_n\|+\|y\|),
\end{equation*}
we get
\begin{align*}
D_{-\tilde{x}} f(x) &\leq \liminf_{n\to \infty} \Big[-2^{-d} \tilde{x}^\top x_n+
C\eps(1+\|x_n\|+\|y\|)\Big] \\
&=-2^{-d}\|x\|+C\eps(1+\|x\|+\|y\|).
\end{align*}
For $\eps>0$ sufficiently small and $C'>0$ sufficiently large, this implies
$D_{-\tilde{x}}f(x)<0$ whenever $\|x\| \geq C'\eps(1+\|y\|)$.
\end{proof}

\section{Comment on Projected-Gradient Surfing}

The projected-gradient surfing algorithm
performs an exhaustive search over pieces
$P_g \in \mathcal{P}(x_{t-1},\theta(\delta t),\tau)$.
The number of such pieces is at most $1+2^{|S(x_{t-1},\theta(\delta t),\tau)|}$,
where we recall that
\begin{equation*}
  S(x,\theta,\tau)=\{(i,j):|w_{i,j}^\top x_{i-1}+b_{i,j}| \leq \tau\}
\end{equation*}
is the collection of layers and rows where the sign could change during the next step.

We reason heuristically that if $\theta \equiv \theta(\delta t)$ is
``generic'', then for sufficiently small $\tau$, we should have
$|S(x,\theta,\tau)| \leq dk$ for all $s \in [0,S]$ and $x \in \R^k$,
so that this search is tractable for small $k$.
Indeed, for fixed $W_1,b_1,\ldots,W_i,b_i$,
the set of possible outputs $\{x_i:x \in \R^k\}$ at
the $i^\text{th}$ layer is a finite union of affine linear
spaces of dimension $k$. For generic $W_{i+1}$ and $b_{i+1}$,
and every $J \subset [n_i]$ where $|J|=k+1$,
each such space has empty intersection with the affine linear space
\begin{equation*}\{z \in \R^{n_i}:w_{i+1,j}^\top z+b_{i+1,j}=0 \text{ for all } j \in J\}\end{equation*}
of dimension $n_i-k-1$. Thus
\begin{equation*}\sup_{x \in \R^k} |\{j \in [n_i]:w_{i+1,j}^\top x_i+b_{i+1,j}=0\}| \leq k,\end{equation*}
so $\sup_{x \in \R^k} |S(x,\theta,0)| \leq dk$ for $\tau=0$. Then we expect
this to hold also for some small $\tau>0$.

%% file: discuss.tex

\section{Discussion}\label{sec:discuss}

This paper has explored the idea of incrementally optimizing a sequence of
objective risk functions obtained for models that are slowly changing during
stochastic gradient descent during training. When initialized with random
parameters $\theta_0$, we have shown that the empirical risk function $ f_{\theta_0}(x) = \frac{1}{2}\|G_{\theta_0}(x) - y\|^2$ is well behaved and easy to optimize. The surfing algorithm initializes $x$ for the current network $G_{\theta_t}(x)$ at the optimum $x^*_{t-1}$ found for the previous network $G_{\theta_{t-1}}(x)$ and then carries out gradient descent to obtain the updated point $x^*_{t} = \argmin_x f_{\theta_t}(x)$. Our experiments show that this scheme has merit, and often significantly outperforms direct gradient descent on the final model alone.

On the theoretical side, our main technical result applies and extends ideas of \cite{hand2017global} to show that for random ReLU networks that are sufficiently expansive, the surface of $f_{\theta_0}(x)$ is well-behaved for arbitrary target vectors $y$. This result may be of independent interest, but it is essential for the surfing algorithm because initially the model is poor, with high approximation error.
The analysis for the incremental scheme uses projected gradient descent,
although we find that simple gradient descent works well in practice. The
analysis assumes that the $\argmin$ over the surface evolves continuously in
training. This assumption is necessary---if the global minimum is discontinuous
as a function of $t$, so that the minimizer ``jumps'' to a far away point, then
the surfing procedure will fail in practice.

In our experiments, we see that simple surfing can indeed be effective for
mapping outputs $y$ to inputs $x$ for the trained network, where it often
outperforms direct gradient descent for a range of deep network architectures
and training procedures. However, these simulations also point to the fact that
in some settings, direct gradient descent itself can be surprisingly effective.
A deeper understanding of this phenomenon could lead to more advanced surfing algorithms that are able to ride
to the final optimum even more efficiently and often.

%% file: ack.tex
\section*{Acknowledgment}
Research supported in part by ONR grant N00014-12-1-0762, NSF grant DMS-1513594,
and NSF grant CCF-1839308.